\documentclass{article}
\usepackage{preprint}

\usepackage{tikz}
\usepackage{amsmath, amssymb, amsfonts, amsthm}
\usepackage{booktabs}
\usepackage{algorithm}
\usepackage{algpseudocode}
\usepackage{multirow}
\usepackage{graphicx}
\usepackage{caption}
\usepackage{enumitem}
\newtheorem{assumption}{Assumption}
\newtheorem{lemma}{Lemma}   
\newtheorem{proposition}{Proposition}
\newtheorem{theorem}{Theorem}

\usepackage{xurl}
\usepackage[hidelinks]{hyperref}
\usepackage{cleveref}
\crefname{assumption}{assumption}{assumptions}
\Crefname{assumption}{Assumption}{Assumptions}
\hypersetup{
    pdftitle={Direct Optimization of Generators for Search in Automated Theorem Proving},
    pdfauthor={Adam Ousherovitch and Ambuj Tewari}
}
\title{Direct Optimization of Generators for Search in Automated Theorem Proving}
\author{%
Adam Ousherovitch \\
Department of Statistics \\
University of Michigan \\
\texttt{aoushero@umich.edu}
\And
Ambuj Tewari \\
Department of Statistics \\
University of Michigan \\
\texttt{tewaria@umich.edu}
}
\date{}

\begin{document}

\maketitle

\begin{abstract}
Fine-tuned Large Language Models (LLMs) significantly advance Automated Theorem Proving (ATP), but are often deployed as guiding policies within tree search rather than for single-attempt generation. Recent work shows cross entropy is suboptimal for an LLM used in flat search strategies such as aggregation or filtering and that work has developed new loss functions to correct this misalignment. Extending this alignment to tree search is more challenging: proof discovery depends on exploration and recovery through off-trace states that supervised demonstrations do not reveal. We extend Compute-Aligned Training (CAT) to this setting through an abstraction of policy-guided search, deriving tractable, trace-supported losses. Alongside these search-aware losses, we introduce a search-agnostic uniform-allocation (UA) loss that accounts for the budget without specifying the specific search. Both induce scalar weights on per-tactic cross-entropy gradients. We characterize how off-trace behavior affects the search-aware weights, including conditions for vanishing approximation error at large budgets. On a Lean benchmark, both approaches achieve higher observed proof-success rates than cross-entropy across six search strategies, with strong results from a single shared UA adapter. Budget sweeps show larger gains over cross-entropy at $256$ than at $16$ expansions, implying CAT scales with test time compute.\looseness-1
\end{abstract}

\section{Introduction}
Automated Theorem Proving (ATP) is an important frontier in Artificial Intelligence (AI) \citep{yang2024formalmathematicalreasoningnew}, with applications in software and hardware verification \citep{bengio2024machinelearninginformationtheory}. Formal verifiers such as Lean~4 \citep{deMoura2021} allow generated proofs to be checked, but proof generation remains difficult because rewards are sparse and horizons are long \citep{poesia2024learningformalmathematicsintrinsic}. Large Language Models (LLMs) predict proof steps \citep{polu2020generativelanguagemodelingautomated}, yet a single flawed step can invalidate an entire attempt. Recent approaches therefore pair LLMs with structured search \citep{lample2022hypertree,polu2022formal,yang2023leancopilot}, which retains partial proofs and explores alternative continuations rather than requiring every step to succeed on the first attempt.

LLMs driving search are typically trained through Supervised Fine-Tuning (SFT) with Cross-Entropy (CE). Recent work shows that CE can be suboptimal when deployment uses search \citep{ousherovitch2026test,chen2025rethinking}, but existing corrections focus on flat strategies such as candidate aggregation \citep{wang2023selfconsistency} or filtering \citep{chen2021codex,cobbe2021verifiers}. Extending alignment to tree search is fundamentally more difficult because success probability includes unsuccessful exploration followed by recovery, and depends on how compute is allocated across intermediate proof states. An exact objective therefore requires off-trace continuations that demonstrated proof traces do not supply.

We address this gap by formalizing policy-guided search and constructing tractable Compute-Aligned Training (CAT) losses for this kind of search. Given a demonstrated proof trace, we restrict the modeled search to observed states and collapse deviations into a trace-miss event. This yields \emph{search-aware} objectives that preserve selected aspects of the strategy's compute allocation. Their gradients are weighted sums of the usual per-tactic CE gradients, emphasizing steps according to their influence on modeled proof success.

Alongside these objectives, we introduce a simple \emph{search-agnostic} uniform-allocation (UA) loss. UA assigns a uniform local compute budget to each demonstrated step, accounting for repeated sampling without specifying the deployed search rule. It thus provides a shared training objective for use across search algorithms, rather than an approximation tied to one particular deployed strategy.

Trace support introduces approximation error because demonstrations omit alternative proofs and the compute consumed by off-trace exploration. We characterize these effects through the per-tactic gradient weights and identify when search-aware weights scale better than CE with more budget.

On a Lean benchmark, both strategy-specific CAT and UA attain higher accuracy than an epoch-matched CE control under all six evaluated strategies (\Cref{tab:matrix}). A shared UA adapter performs strongly across strategies, while search-aware CAT yields an additional observed gain under Pass@$N$. Experiments also show larger gains over CE at $N=256$ than at $N=16$ (\Cref{tab:sweep}).\looseness-1

\subsection{Related Work}

\textbf{Automated Theorem Proving.} The application of deep learning to ATP has rapidly evolved alongside LLMs \citep{polu2020generativelanguagemodelingautomated, han2021proof}. Modern neural provers target interactive verification environments like Lean \citep{deMoura2021, zheng2022minif2f}, with open-weight milestones such as Lean Copilot \citep{yang2023leancopilot}, Llemma \citep{azerbayev2023llemma}, and ReProver \citep{song2025leancopilotlargelanguage} powered almost exclusively by SFT. Our work is strictly situated within this SFT paradigm. We explicitly scope our analysis away from systems that interleave informal natural-language reasoning with formal verification \citep{jiang2023draft, wu2022autoformalization}, as well as online Reinforcement Learning (RL) pipelines \citep{deepmind2024alphaproof}.

\textbf{Test-Time Search Algorithms.} The use of search in ATP falls in the wider trends of utilizing inference-time compute \citep{snell2024scaling}. Search in ATP is diverse. It includes uninformed traversal methods like Depth-First Search (DFS) \citep{pearl1984heuristics} or Pass@N \citep{chen2021codex}, and the more complex policy-guided searches like Best-First Search (BFS) \citep{hart1968formal}, most notably Levin Search and its modern neural adaptations \citep{orseau2018singleagentpolicytreesearch, orseau2021policyguidedheuristicsearchguarantees, xin2025bfsproverscalablebestfirsttree}. While prior literature has extensively benchmarked these algorithms in isolation, it treats the language model as a static, search-agnostic prior. Crucially, the literature has never investigated how the structural assumptions of these distinct search rules interact with the training of the policy itself. \looseness=-1

\textbf{Test-Time Aligned Training.} A growing body of work recognizes that standard supervised objectives fail to align models with inference-time search \citep{ousherovitch2026test, chen2025rethinking, chow2025inferenceaware, balashankar2025infaligninferenceawarelanguagemodel, tang2025optimizing}. However, these efforts focus overwhelmingly on non-branching strategies where the global search objective can be expressed as an exact function of observed generation probabilities. Within ATP, the sole existing attempt at test-time aligned training is Direct Coverage Optimization (DCO) \citep{chen2025rethinking}, which focused exclusively on Pass@N. Because the global success probability of Pass@N can be expressed as a closed-form function of on-path tactic probabilities, DCO sidesteps the structural approximation problem entirely. By contrast, state-of-the-art theorem provers rely on branching strategies whose objectives cannot be resolved purely from on-path data. This paper bridges that gap.

\subsection{Contributions}
\begin{itemize}[leftmargin=*,topsep=0pt]
    \item \textbf{A Framework for Search-Aligned Training.} We formalize policy-guided search and derive tractable losses for branching strategies. We also develop a budget-based UA loss independent of the specific search's structure. Both approaches induce scalar weights on per-tactic CE gradients.

    \item \textbf{Characterizing the Trace-Supported Approximation.} For Pass@$N$ and best-first search, we analyze how unobserved alternative proofs and off-trace compute consumption affect search-aware gradient weights, including their dependence on the test-time budget.

    \item \textbf{Empirical Evaluation in Lean.} Across six search strategies, strategy-specific CAT and UA each attain higher success rates than CE. We also evaluate gains across deployment budgets.
\end{itemize}

\section{A General Objective for Search-Aligned Training}
\label{sec:formal_cat}

This section formalizes LLM-guided search and derives a general aligned objective and the search-agnostic UA objective. We assume familiarity with Interactive Theorem Proving (ITP).\footnote{A full formalization is provided in \Cref{app:mdp_formulation}.}

\subsection{Setup}
\label{subsec:setup}

Let the environment be a state space $\mathcal{X}$ (proof states), an action space $\mathcal{Y}$, and a transition $\mathcal{T}:\mathcal{X}\times\mathcal{Y}\to\mathcal{X}\cup\{x_{err}\}$, where $x_{err}$ is the state produced by incorrect tactics. The LLM defines a stochastic policy $\pi_\theta(y\mid x)$. A dataset $\mathcal{D}$ provides valid traces; a single trace is
\[
    \tau^* = (x_0, y_1^*, x_1, \ldots, x_{L-1}, y_L^*, x_L),
    \qquad x_L = x_{QED},
\]
where each demonstrated tactic $y_t^*$ moves the environment from $x_{t-1}$ to $x_t$. Write $p_t \;=\; \pi_\theta(y_t^* \mid x_{t-1})$. Standard SFT with CE maximizes the likelihood of these traces:
\vspace{-6pt}
\begin{equation}
    L_{CE}(\theta)
    = -\sum_{t=1}^{L} \log \pi_\theta(y_t^* \mid x_{t-1})
    = -\sum_{t=1}^{L} \log p_t,
    \qquad
    L_{CE}^{(t)} := -\log p_t.
    \label{eq:sft_loss}
    \vspace{-6pt}
\end{equation}

At test time, the policy is deployed inside a search algorithm $\mathcal{A}$ with a compute budget. We wish to optimize the probability that $\mathcal{A}$, using $\pi_\theta$, finds a proof within that budget.

\subsection{Search as Configuration Dynamics}
\label{subsec:search_dynamics}

$\mathcal{A}$ is defined by its interaction with memory, which we define via \emph{configurations}: \(
    c = \langle \mathcal{F}, M \rangle .
\)

The frontier $\mathcal{F}$ holds the nodes available for expansion with a heuristic score (insertion order, internal confidence, etc.); its elements are tuples $(x,s)$ with $x\in\mathcal{X}$, $s\in\mathbb{R}$. We may write $x\in\mathcal{F}$ for $(x,s)\in\mathcal{F}$ for notational ease. $M$ holds everything else: the policy, the \emph{remaining budget}, the unavailable nodes, etc. Budget is most naturally interpreted as either tactic expansions or rollouts. Budget handling can be encoded in the specification of the strategy.

$\mathcal{A}$ selects a node $x\in\mathcal{F}$, samples a tactic $y \sim \pi_\theta(\cdot\mid x)$, executes it, and updates its configuration from the observation $z=\mathcal{T}(x,y)$. Excluding the initialization of $c_0$ from $x_0$ and the compilation of the returned proof from $M$, $\mathcal{A}$ is fully specified by two operations,
\[
    S_{\mathcal{A}}(x \mid c), \qquad U_{\mathcal{A}}(c, x, z),
\]
a (possibly stochastic) selection rule $S_{\mathcal{A}}$ and a deterministic update rule $U_{\mathcal{A}}$.

A configuration is \emph{solved} if it contains a completed proof, and \emph{dead} if its budget is exhausted. Let $J^{\mathcal{A}}(c)$ be the true probability that $\mathcal{A}$, started from $c$, ever discovers a proof. It satisfies
\begin{equation}
\label{eq:exact_search_recurrence}
    J^{\mathcal{A}}(c) =
    \begin{cases}
        1 \text{ if } c \text{ solved}; \quad 0 \text{ if } c \text{ dead (and unsolved)}, \\
        \sum_{x\in\mathcal{F}(c)} S_{\mathcal{A}}(x\mid c) \sum_{y\in\mathcal{Y}} \pi_\theta(y\mid x)\, J^{\mathcal{A}}\!\big(U_{\mathcal{A}}(c,x,\mathcal{T}(x,y))\big), & \text{otherwise}
    \end{cases}
\end{equation}
and the deployed objective is $\mathcal{J}^{\mathcal{A}}(\theta) = J^{\mathcal{A}}(c_0)$, where $c_0$ encodes the start state $x_0$ and the full budget. Returning any valid trajectory from $x_0$ to $x_{QED}$ counts as success.

\Cref{eq:exact_search_recurrence} is exact but cannot be evaluated from a supervised trace. Consider expanding a frontier node on the trace, say $x=x_{t-1}$. The tactic space partitions into three disjoint sets: the demonstrated tactic $\{y_t^*\}$, the invalid tactics $\mathcal{I}_t := \{y : \mathcal{T}(x_{t-1},y)=x_{err}\}$, and the valid but off-trace tactics $\mathcal{O}_t := \{y : \mathcal{T}(x_{t-1},y)\notin\{x_{err},x_t\}\}$. With off-trace states defined as $x_{t,y}=\mathcal{T}(x_{t-1},y)$ for $y\in\mathcal{O}_t$, the sum in \Cref{eq:exact_search_recurrence} then expands to:
\begin{equation}
\label{eq:local_exact_decomposition}
\begin{split}
    &\sum_{y} \pi_\theta(y\mid x_{t-1})\, J^{\mathcal{A}}\!\big(U_{\mathcal{A}}(c,x_{t-1},\mathcal{T}(x_{t-1},y))\big) = p_t\, J^{\mathcal{A}}\!\big(U_{\mathcal{A}}(c,x_{t-1},x_t)\big) \\[-0.5ex]
    &+ \sum_{y\in\mathcal{I}_t} \pi_\theta(y\mid x_{t-1})\, J^{\mathcal{A}}\!\big(U_{\mathcal{A}}(c,x_{t-1},x_{err})\big) + \sum_{y\in\mathcal{O}_t} \pi_\theta(y\mid x_{t-1})\, J^{\mathcal{A}}\!\big(U_{\mathcal{A}}(c,x_{t-1},x_{t,y})\big).
\end{split}
\end{equation}
The dataset supplies neither probabilities within $\mathcal{I}_t\cup\mathcal{O}_t$ nor continuations $J^{\mathcal{A}}(U_{\mathcal{A}}(c,x_{t-1},x_{t,y}))$ containing unobserved off-trace states. Off-trace tactics may lead to alternate proofs or rejoin the trace so the exact objective depends on the search space outside of our dataset.

\subsection{The Trace-Supported Surrogate}
\label{subsec:trace_supported}

To obtain a tractable objective, we restrict the search to states from the demonstrated trace. Let $\widehat{\mathcal{C}}_{\tau^*}$ be the configurations whose frontier and memory reference only $x_0,\ldots,x_L$. We replace the true transition by an offline version,
\begin{equation}
    \widehat{\mathcal{T}}_{\tau^*}(x_{t-1},y) =
    \begin{cases}
        x_t,        & y = y_t^*, \\
        \bot_\tau,  & y \neq y_t^*,
    \end{cases}
    \vspace{-8pt}
    \label{eq:trace_transition}
\end{equation}
\vspace{-8pt}

where $\bot_\tau$ marks any deviation from the trace: valid off-trace and invalid tactics are collapsed into one training-time event. This transition induces two projected update rules,
\[
    \widehat{U}_{\mathcal{A}}^{+}(\widehat{c}, x),
    \qquad
    \widehat{U}_{\mathcal{A}}^{-}(\widehat{c}, x),
\]

the positive update when the selected node advances $x_{t-1}\to x_t$, and the negative update following a trace miss. Both updates map to $\widehat{\mathcal{C}}_{\tau^*}$. At a demonstrated state, the model either samples $y_t^*$ with probability $p_t$ or deviates with probability $1-p_t$.

\paragraph{Trace-supported recurrence.}
The surrogate success probability is
\begin{equation}
\label{eq:trace_supported_recurrence}
    \widehat{J}_{\mathcal{A}}(\widehat{c}; p) =
    \begin{cases}
        1 \text{ if } \widehat{c} \text{ solved}; \quad 0 \text{ if } \widehat{c} \text{ dead}, \\
        \sum_{x\in\widehat{\mathcal{F}}(\widehat{c})} {S}_{\mathcal{A}}(x\mid\widehat{c};p) \Big[ p_{t(x)}\, \widehat{J}_{\mathcal{A}}\!\big(\widehat{U}_{\mathcal{A}}^{+}(\widehat{c},x);p\big) + (1-p_{t(x)})\, \widehat{J}_{\mathcal{A}}\!\big(\widehat{U}_{\mathcal{A}}^{-}(\widehat{c},x);p\big) \Big], & \text{else}
    \end{cases}
\end{equation}
where $t(x)$ is the index of $x$ in $\tau^*$.

We call this loss Compute Aligned Training (CAT), following \citet{ousherovitch2026test}
\begin{equation}
\widehat{\mathcal{J}}_{\tau^*}^{\mathcal{A}}(\theta)
    = \widehat{J}_{\mathcal{A}}\!\big(\widehat{c}_0;\, p(\theta)\big),
    \qquad
    \widehat{L}_{CAT}^{\mathcal{A}}(\theta)
    = -\log \widehat{\mathcal{J}}_{\tau^*}^{\mathcal{A}}(\theta),
    \label{eq:trace_supported_loss}
\end{equation}
where $\widehat{c}_0$ encodes $x_0$ and the initial budget. We will assume the recursion is finite.\footnote{Sufficient conditions for this include decrements of budget every expansion or decrements every rollout and a finite demonstrated trace.}

\subsection{The CAT Gradient}
\label{subsec:gradient}

The following Proposition shows search strategy can be completely specified for training by its triple $(S_{\mathcal{A}}, \widehat{U}_{\mathcal{A}}^{+}, \widehat{U}_{\mathcal{A}}^{-})$, and that its loss and per-tactic gradient weights follow mechanically.

\begin{proposition}[CAT Loss]
\label{prop:cat_template}
Fix a strategy $\mathcal{A}$ given by $(S_{\mathcal{A}},\widehat{U}_{\mathcal{A}}^{+}, \widehat{U}_{\mathcal{A}}^{-})$. Then:
\begin{enumerate}
    \item[(i)] \Cref{eq:trace_supported_recurrence} has a unique solution, and $\widehat{\mathcal{J}}_{\tau^*}^{\mathcal{A}}(\theta)$ depends on $\theta$ only through $p=(p_1,\ldots,p_L)$.\footnote{Differentiability naturally follows if selection is smooth in $\theta$.}
    \item[(ii)] The CAT gradient is a weighted sum of per-step cross-entropy gradients,
    \begin{equation}
        \nabla_\theta \widehat{L}_{CAT}^{\mathcal{A}}(\theta) = \sum_{t=1}^{L} \widehat{w}_t^{\mathcal{A}}    \nabla_\theta L_{CE}^{(t)},
        \qquad
        \widehat{w}_t^{\mathcal{A}}
        = \frac{p_t}{\widehat{\mathcal{J}}_{\tau^*}^{\mathcal{A}}}
          \,\frac{\partial \widehat{\mathcal{J}}_{\tau^*}^{\mathcal{A}}}{\partial p_t}.
        \label{eq:generic_cat_gradient}
    \end{equation}
    \vspace{-16pt}
\end{enumerate}
\end{proposition}

\begin{proof}
(i) Under the assumption the budget is finite and decrements with each update, it terminates after finitely many steps and assigns a unique value to each configuration. All $\theta$-dependence enters only through $p_{t(v)}$ or possibly through the scores $s_{\mathcal{A}}(\cdot;p)$.

(ii) By (i), $\nabla_\theta \widehat{\mathcal{J}}_{\tau^*}^{\mathcal{A}} = \sum_t \partial_{p_t}\widehat{\mathcal{J}}_{\tau^*}^{\mathcal{A}}\, \nabla_\theta p_t$. From $L_{CE}^{(t)}=-\log p_t$ we have $\nabla_\theta p_t = -p_t\,\nabla_\theta L_{CE}^{(t)}$, so
\[
    \nabla_\theta\big[-\log\widehat{\mathcal{J}}_{\tau^*}^{\mathcal{A}}\big]
    = -\frac{1}{\widehat{\mathcal{J}}_{\tau^*}^{\mathcal{A}}}
        \sum_{t=1}^{L}
        \frac{\partial \widehat{\mathcal{J}}_{\tau^*}^{\mathcal{A}}}{\partial p_t}
        \big(-p_t\,\nabla_\theta L_{CE}^{(t)}\big)
    = \sum_{t=1}^{L}
        \frac{p_t}{\widehat{\mathcal{J}}_{\tau^*}^{\mathcal{A}}}
        \frac{\partial \widehat{\mathcal{J}}_{\tau^*}^{\mathcal{A}}}{\partial p_t}
        \,\nabla_\theta L_{CE}^{(t)}. \qedhere
\]
\end{proof}

As noted by \citet{chen2025rethinking} for the case of Pass@$N$, separating $\widehat{w}_t^{\mathcal{A}}$ lets one filter low-weight gradients before they corrupt batch normalization statistics, stabilizing training.

\subsection{A Search-Agnostic Uniform-Allocation Objective}
\label{subsec:mf_main}

The same gradient form permits a compute-aware objective without specifying $\mathcal A$. For an expansion budget $N\ge L$, UA assigns each demonstrated step a uniform local budget $\bar K=N/L$ and uses
\begin{equation}
    \widehat{\mathcal J}^{\mathrm{UA}}(p;N)
    =\prod_{t=1}^{L}\big[1-(1-p_t)^{\bar K}\big],
    \qquad
    \widehat L_{CAT}^{\mathrm{UA}}
    =-\log\widehat{\mathcal J}^{\mathrm{UA}}.
    \label{eq:mf_objective_main}
\end{equation}
For integer $\bar K$, this is the probability that every demonstrated step succeeds within its own quota of independent attempts, without transferring unused quota between steps. For noninteger $N/L$, we use the continuous extension of the formula. This allocation is a modeling choice, not a claim about how a particular deployed algorithm distributes its budget. By \Cref{prop:cat_template},
\begin{equation}
    \widehat w_t^{\mathrm{UA}}
    =\frac{\bar Kp_t(1-p_t)^{\bar K-1}}
           {1-(1-p_t)^{\bar K}}.
    \label{eq:mf_weight_main}
\end{equation}
The rule depends on $N$, $L$, and $p_t$, but not on $\mathcal{A}$. At $N=L$ it recovers CE; additional local attempts reduce the emphasis on tactics already likely to be sampled. Thus UA is search-agnostic, but not compute-agnostic. Its connection to the shared-budget formulation is detailed in \Cref{app:meanfield}.

\section{Strategy-Specific Objectives}
\label{sec:strategies}
We instantiate the construction for Pass@$N$ and best-first search using \Cref{prop:cat_template}.\footnote{Objectives for other strategies are given in \Cref{app:DFS,sec:mcts}.}  

\subsection{Pass@\texorpdfstring{$N$}{N}}
\label{subsec:passn}

Pass@$N$ uses up to $N$ independent rollouts (\Cref{fig:pass_at_n}); here $N$ counts \emph{rollouts}, not tactic expansions. In the trace-supported model, let $\widehat c=(r,t)$, where $r$ counts remaining rollouts, including the current one, and $t$ indexes the next demonstrated tactic. Advancement gives $(r,t+1)$; a miss abandons the rollout and gives $(r-1,1)$. Thus \Cref{eq:trace_supported_recurrence} becomes
\begin{equation}
    \widehat J(r,t)
    =p_t\widehat J(r,t+1)+(1-p_t)\widehat J(r-1,1),
    \label{eq:passn_recurrence}
\end{equation}
with $\widehat J(r,L+1)=1$ and $\widehat J(0,t)=0$ for $t\le L$. A rollout completes the demonstrated trace with probability $\pi_{\mathrm{proof}}:=\prod_{t=1}^{L}p_t$, giving
\begin{equation}
    \widehat{\mathcal J}_{\tau^*}^{\mathrm{Pass@}N}(\theta)
    =1-(1-\pi_{\mathrm{proof}})^N.
    \label{eq:passn_objective}
\end{equation}
Applying \Cref{prop:cat_template} yields the same weight at every step of the trace:
\[
    \widehat w_t^{\mathrm{Pass@}N}
    =\frac{N\pi_{\mathrm{proof}}(1-\pi_{\mathrm{proof}})^{N-1}}
           {1-(1-\pi_{\mathrm{proof}})^N}.
\]

\subsection{Best-First Search}
\label{subsec:bfs}

We consider policy-guided best-first search (BFS) \citep{orseau2018singleagentpolicytreesearch,orseau2021policyguidedheuristicsearchguarantees,xin2025bfsproverscalablebestfirsttree}. A node $x_k$ reached through tactics $y_1,\ldots,y_k$, $k\ge1$, is scored by
\begin{equation}
    s_{\mathrm{BFS}}(x_k;\pi_\theta)
    =\left(\prod_{j=1}^{k}\pi_\theta(y_j\mid x_{j-1})\right)^{1/k}.
    \label{eq:bfs_score}
\end{equation}
The algorithm pops the highest-scoring frontier node, samples $B$ tactics there, and inserts valid children (\Cref{fig:levin_search}). Here $N$ counts \emph{tactic expansions}: each sample costs one unit.

\paragraph{Trace-supported instantiation.}
We use a retry surrogate with configuration $\widehat c=(b,t)$, where $b$ is the remaining expansion budget and the only live node is $x_{t-1}$. Starting from $(N,1)$, advancement gives $(b-1,t+1)$ and a miss gives $(b-1,t)$. Retaining the node after a miss is a modeling choice: this surrogate replaces the deployed finite per-pop sampling with retries until the demonstrated tactic is sampled or the budget is exhausted. Its singleton frontier also removes score-based competition. The recurrence is
\begin{equation}
\begin{split}
    \widehat J_{\mathrm{BFS}}(b,t)
    &=p_t\widehat J_{\mathrm{BFS}}(b-1,t+1)
      +(1-p_t)\widehat J_{\mathrm{BFS}}(b-1,t),\\
    \widehat J_{\mathrm{BFS}}(b,L+1)&=1,\qquad
    \widehat J_{\mathrm{BFS}}(0,t)=0\quad(t\le L).
\end{split}
\label{eq:bfs_trace_recurrence}
\end{equation}

\paragraph{Objective.}
Let $T_t\sim\mathrm{Geom}(p_t)$ independently count the samples needed at step $t$, including the successful sample, and write $S_L=\sum_{t=1}^{L}T_t$. For $N\ge L$, the surrogate success probability is
\begin{equation}
    \widehat{\mathcal J}_{\tau^*}^{\mathrm{BFS}}(\theta)
    =\mathbb P(S_L\le N).
    \label{eq:bfs_trace_objective}
\end{equation}
The shared budget couples the per-tactic weights, which follow by differentiating \Cref{eq:bfs_trace_recurrence}. The connection to UA's fixed local allocation is detailed in \Cref{app:meanfield}.

\section{Approximation Error}
\label{sec:main_theory}

A demonstrated trace tells us how one proof succeeds, but not what happens when search leaves that path. At state $x_{t-1}$, we can compute the probability of missing the demonstrated tactic. This section intends to characterize what is lost in not having information about this event at training. We study this missing information through two channels: alternative proofs (the \emph{bypass channel}) and compute spent on unsuccessful off-trace exploration (the \emph{trap channel}). How do these unseen consequences change the importance of sampling the demonstrated tactic? Full derivations are in \Cref{sec:scaling}.\footnote{Throughout, $p_s\in(0,1)$ and $w_t(J)=p_t\partial_{p_t}\log J$ for differentiable $J>0$. Derivatives follow \Cref{eq:reparam}: all other demonstrated probabilities, conditional distributions over non-demonstrated tactics, and off-trace policy conditionals are held fixed.}

\subsection{Effects of the Channels on the True Weight}
\label{subsec:Channels}

\paragraph{Bypass: missing the trace isn't failing.} A different tactic may lead to a valid proof without completing $\tau^*$. The surrogate cannot see these successes. For Pass@$N$, accounting for alternative proofs reduces the deployed weight relative to the trace-supported weight.

\begin{theorem}[Bypass Deflation]
\label{thm:main_bypass}
Under \Cref{assum:single_visit}, for Pass@$N$ with $N>1$ and every step $t$,
\begin{equation}
    w_t^\star
    \le \widehat w_t^{\mathrm{Pass@}N}
    < 1 = w_t^{\mathrm{CE}}.
    \label{eq:main_passn_squeeze}
\end{equation}
Consequently,
$|\widehat w_t^{\mathrm{Pass@}N}-w_t^\star|<|1-w_t^\star|$.
\end{theorem}
The proof and decomposition of the missing alternative-success contributions are given in \Cref{thm:passn_squeeze,eq:passn_error_decomp}.

\paragraph{Trap: a miss can cost much more than one attempt.} A valid but unhelpful tactic can lead search into a branch that consumes many expansions before the demonstrated state is selected again. Unlike an immediate rejection, this excursion leaves less budget for completing the proof. Our approximation understates the cost of such misses.

To isolate this effect, we censor alternative proofs, so only completion of $\tau^*$ counts as success while off-trace exploration remains possible (\Cref{lem:censoring}). We quantify the resulting upward pressure on weights in a model with a common deterministic return cost.

\begin{theorem}[Shared-Budget Bracketing under Deterministic Excursions]
\label{thm:main_trap}
Under \Cref{assum:effective_deviation_cost_revised}, let $N\ge L$ be an integer. Every miss costs $1\le\kappa\le\bar\kappa<\infty$ expansions before returning to the same trace state; successful trials cost one, and $\kappa$ is independent of $p$. Then
\begin{equation}
    w_t^{\mathrm{trap}}(N)
    \in\left[
        \widehat w_t(N),\;
        \widehat w_t\!\left(L+\frac{N-L}{\bar\kappa}\right)
    \right]\subseteq[\widehat w_t(N),1],
    \label{eq:main_trap_bracket}
\end{equation}
where $w_t^{\mathrm{trap}}=w_t(J^{\mathrm{trap}})$ and noninteger budget arguments are rounded down.
\end{theorem}

The mechanism is a reduction in effective budget to $N_{\mathrm{eff}}=L+(N-L)/\kappa$  (\Cref{prop:shared_budget_bracketing_revised}). The surrogate assumes $\kappa=1$: a miss costs one expansion. At the other extreme, if every miss prevents recovery, success requires following the trace without mistakes, recovering CE's unit weights. Finite excursion costs interpolate between these cases within this model.

\paragraph{What the trace cannot determine.} Two environments can agree on every demonstrated transition, all $p_s$, the search rule, and the budget, yet differ in what a miss causes. When those off-trace differences produce distinct deployed weights, no rule using only these trace-level inputs can be exact in both environments (\Cref{prop:irreducible_ambiguity}). Thus, solving the trace-supported recurrence exactly does not resolve the missing information about recovery.

\subsection{Large-Compute Scaling: Can Search Recover from a Miss?}
\label{subsec:main_scaling}

More compute can compensate for a delay, but not for permanent loss of the opportunity to recover. At $N=L$ in the trap-only model, there is no room for wasted expansions: success requires every demonstrated tactic on its first attempt, and both the deployed and surrogate weights equal one. Larger budgets create opportunities to recover, making the outcome depend on whether excursions eventually return. The following limits hold with the policy and excursion parameters fixed.

\begin{theorem}[Large-Budget Limits and Recoverability]
\label{thm:main_scaling}
\begin{enumerate}[leftmargin=*,topsep=0pt]
    \item[(i)] \textbf{Pass@$N$.}
    Under \Cref{assum:single_visit},
    \[
        w_t^\star(N)\longrightarrow 0,
        \qquad \widehat w_t(N)\longrightarrow 0.
    \]
    Hence the search-aware weight error tends to zero, while CE's weight error tends to one.

    \item[(ii)] \textbf{Shared-budget excursions.}
    Under \Cref{assum:absorbing_excursions}, each miss at step $s$ independently causes permanent absorption with probability $q_s\in[0,1)$; otherwise it incurs a bounded excursion before returning to the same trace state. Only the demonstrated proof counts as success. Holding $q_s$ and the finite-excursion laws fixed under differentiation,
    \[
        \widehat w_t(N)\longrightarrow 0,
        \qquad w_t^{\mathrm{trap}}(N)\longrightarrow
        \omega_t:=\frac{q_t}{p_t+(1-p_t)q_t}.
    \]
\end{enumerate}
\end{theorem}
See \Cref{prop:passn_asymptotics,prop:hump_floor} for the derivations.

When all excursions are recoverable ($q_s=0$), both deployed and surrogate weights eventually vanish: avoiding a particular miss matters less when search has enough budget to try again. With permanent absorption, avoiding the miss retains value even with unlimited compute. The retry surrogate overlooks this risk, leaving limiting error $\omega_t$; CE's limiting error is $1-\omega_t$.

The analysis therefore identifies what trace-only training leaves unresolved: alternative routes to success, the compute needed to recover, and the possibility of never recovering. These quantities can be modeled, but are not determined by the demonstrated trace. Additional assumptions can specify them, and training-time search offers a way to collect evidence about them. Neither is replaced by evaluating the same trace-supported surrogate more exactly.

\section{Experiments}
\label{sec:experimental_methodology}

We test whether accounting for budgeted search during SFT improves proof discovery, whether matching the training objective to the deployed strategy helps, and how the resulting gains vary with test-time budget. We compare search-aware CAT and search-agnostic UA with a CE control.

\paragraph{Setup.}
We use Lean~4 through LeanDojo \citep{yang2023leandojo}, training on \texttt{leandojo-\allowbreak benchmark-\allowbreak 4-\allowbreak random}, a random split of mathlib4 \citep{mathlib2020}. The held-out pool contains $458$ theorems whose reference proofs have $2$--$5$ steps; any kernel-verified proof counts as success. We initialize from Qwen2.5-Math-7B-Instruct \citep{qwen2024qwen25} and fine-tune with LoRA \citep{hu2022lora}. Within each experiment, models share a CE warm-up.\footnote{Two warm-up epochs for \Cref{subsec:exp_matrix} and one for \Cref{subsec:exp_sweep}.} The control then receives one additional CE epoch, while the search-aware and UA models use one additional epoch with their respective weights. Data, ordering, and optimizer settings are otherwise identical. Evaluation budgets count node expansions, each comprising one tactic sample and its kernel verification.\footnote{The Pass@$N$ objective in \Cref{subsec:passn} counts rollouts. We use $N/4$ as its training parameter, a fixed approximation to the expansion budget rather than an exact conversion.} Full details are in \Cref{app:experimental_details}.

\subsection{Experiment 1: Training for Search Across Strategies}
\label{subsec:exp_matrix}

At $N=256$ expansions, we evaluate Pass@$N$, BFS, DFS variants (\Cref{app:DFS}), and Monte Carlo Tree Search (MCTS) \citep{kocsis2006bandit}.\footnote{Our MCTS implementation uses the policy's Levin score as its value estimate.} For each deployed strategy, we compare CE, the corresponding search-aware weights, and the same single shared UA adapter. This study also applies wall-clock caps; capped runs without a proof count as failures (\Cref{app:details_inference}).

\begin{table}[ht!]
\centering
\caption{\textbf{Proofs found (\%) at $N=256$ on $458$ held-out theorems.} Columns are deployed strategies. The same UA adapter is used throughout. $^{*}$ marks $p<0.05$ against CE under the same strategy; bold marks the highest reported rate.}
\label{tab:matrix}
\small
\begin{tabular}{lcccccc}
\toprule
Training objective & Pass@$N$ & BFS & DFS & RDFS & VDFS & MCTS \\
\midrule
CE & 21.2 & 23.8 & 19.2 & 20.7 & 21.2 & 20.3 \\
Search-aware CAT
 & \textbf{26.6}$^{*}$ & 25.1 & 20.7 & 23.6$^{*}$ & 23.8 & 22.5 \\
Shared UA
 & 25.1$^{*}$ & \textbf{26.9} & \textbf{23.6}$^{*}$
 & \textbf{25.3}$^{*}$ & \textbf{24.5}$^{*}$ & \textbf{25.1}$^{*}$ \\
\bottomrule
\end{tabular}
\end{table}

\begin{figure}[ht]
\centering
\includegraphics[width=0.8\linewidth]{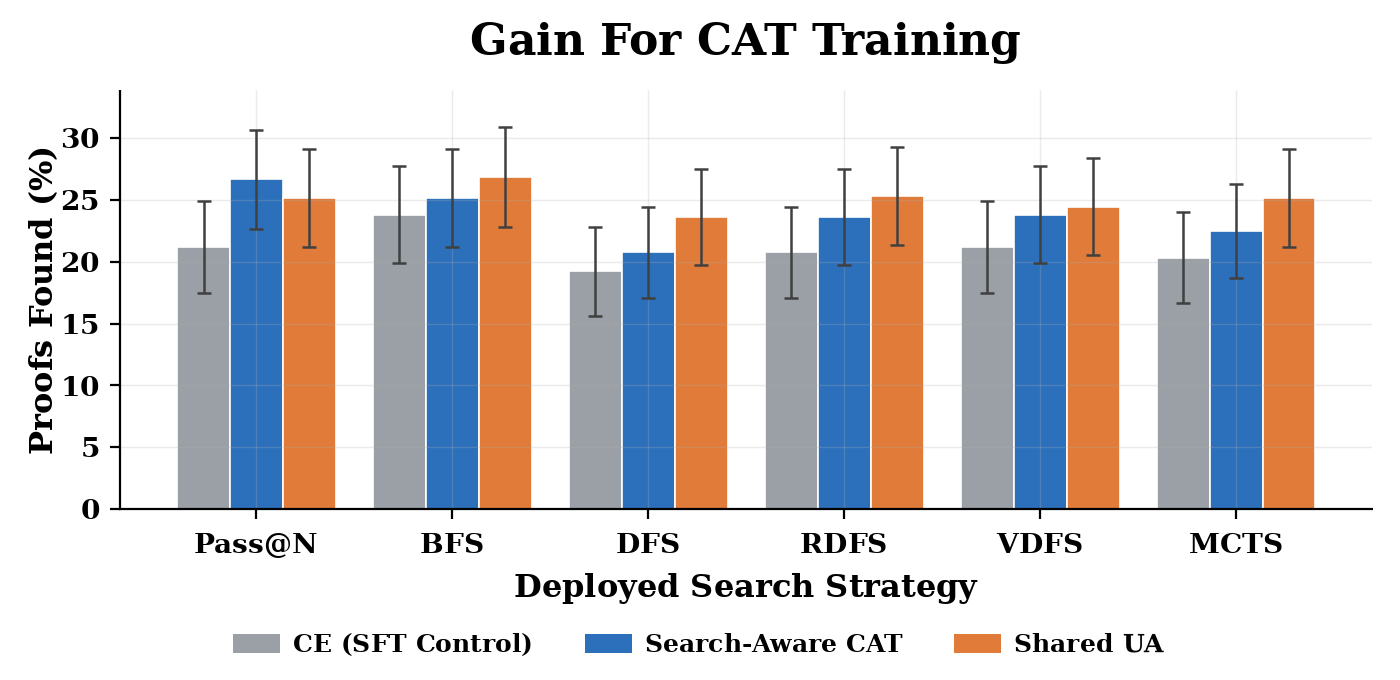}
\caption{\textbf{Both training families improve accuracy over CE.} Proofs found (\%) on $458$ held-out theorems with a budget of $N=256$ tactic expansions. Search-aware CAT uses strategy-matched training objectives. UA uses one adapter across all strategies. UA has the highest accuracy under branching strategies; search-aware CAT leads under Pass@$N$. Error bars show $95\%$ intervals.\looseness-1}

\label{fig:main_comparison}
\vspace{-16pt}
\end{figure}

\paragraph{Results} Both search-aware CAT and UA have higher observed success rates than CE under all six strategies (\Cref{tab:matrix}), showing that accounting for deployment-time compute provides substantial gains over standard SFT. The shared UA objective is a strong search-agnostic approach, attaining the highest point estimate under the five branching strategies. At the same time, search-aware CAT achieves the highest observed success under Pass@$N$ ($26.6\%$ versus $25.1\%$ for UA), showing modeling the deployed search strategy can yield additional improvement. Consistent with this, within the search-aware family, matching training and deployment gives higher point estimates for the Pass@$N$/BFS adapter pair: $26.6\%$ versus $23.7\%$ under Pass@$N$, and $25.1\%$ versus $24.5\%$ under BFS (\Cref{tab:matrix_abl_2}). Together, these results support compute-aware training and suggest that modeling the deployed strategy can provide additional benefits.

\subsection{Experiment 2: Alignment Gain versus Test-Time Budget}
\label{subsec:exp_sweep}

We train a separate search-aware CAT adapter for each strategy (Pass@$N$ or BFS) and budget $N\in\{16,64,256\}$, and evaluate each adapter at its corresponding budget. This sweep removes the per-strategy and per-theorem wall-clock caps, but each model receives at most $80$ hours of evaluation on four A100 GPUs. At each budget, we compare CAT and CE only on theorems for which both completed evaluation; \Cref{tab:sweep} reports the paired subset sizes. $N=1024$ runs completed only $56$--$57$ paired theorems and are reported separately in \Cref{app:sweep_1024}.

\begin{table}[ht]
\centering
\caption{\textbf{Paired search-aware CAT gain over CE.} $n$ is the paired subset size; $b$ and $c$ count CAT-only and CE-only successes. Gains are in percentage points. The $95\%$ intervals use the variance approximation $(b+c)/n^2$ on the probability scale; $p$ is the exact two-sided McNemar test.}
\label{tab:sweep}
\small
\begin{tabular}{llcccc}
\toprule
Strategy & Budget $N$ & $n$ & $b/c$ & Gain (pp), $95\%$ interval & $p$ \\
\midrule
\multirow{3}{*}{Pass@$N$}
 & $16$  & 458 & 10/1  & $+2.0\;[0.5,3.4]$  & $0.012$ \\
 & $64$  & 458 & 28/8  & $+4.4\;[1.8,6.9]$  & $0.001$ \\
 & $256$ & 408 & 29/12 & $+4.2\;[1.1,7.2]$  & $0.012$ \\
\midrule
\multirow{3}{*}{BFS}
 & $16$  & 458 & 28/8  & $+4.4\;[1.8,6.9]$  & $0.001$ \\
 & $64$  & 458 & 34/13 & $+4.6\;[1.7,7.5]$  & $0.003$ \\
 & $256$ & 407 & 37/9  & $+6.9\;[3.6,10.1]$ & $<0.001$ \\
\bottomrule
\end{tabular}
\end{table}

\begin{figure}[ht]
\centering
\includegraphics[width=0.55\linewidth]{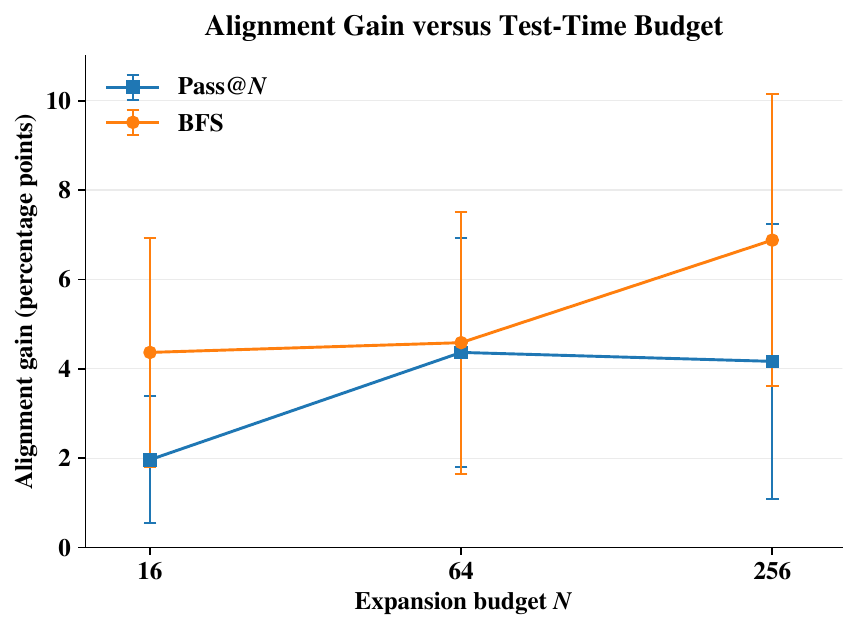}
\caption{Paired search-aware CAT gain over CE with budget-matched adapters. Error bars are $95\%$ intervals. Completed subsets differ across budgets, so connecting lines are descriptive rather than a common-cohort comparison.}
\label{fig:sweep_gain}
\end{figure}

\paragraph{Gains persist across the evaluated budgets.}
Search-aware CAT has positive paired gains at every reported budget in \Cref{tab:sweep}. The observed gains at $N=256$ exceed those at $N=16$ for both Pass@$N$ ($+4.2$ versus $+2.0$ points) and BFS ($+6.9$ versus $+4.4$). However, these are finite-budget performance comparisons, not direct measurements of the weight-error limits in \Cref{sec:main_theory}.

\section{Conclusion}
This paper studies how to train an LLM from demonstrations when it will be deployed within policy-guided search under a finite compute budget. We develop two complementary approaches: search-aware CAT objectives that incorporate a model of the search strategy, and a search-agnostic uniform-allocation objective that accounts for compute without specifying the search rule. Both induce scalar weights on the usual per-tactic CE gradients. Although our experiments focus on ATP, the framework provides a starting point for studying other settings in which generative models guide tree search. Our analysis characterizes how alternative proofs and off-trace exploration affect search-aware gradient weights. It identifies conditions under which their approximation error vanishes as the test-time budget grows while CE remains misaligned. Empirically, all objectives attain higher observed proof-success rates than an epoch-matched CE control across all six evaluated strategies. The shared UA objective performs strongly across search algorithms, while search-aware CAT achieves the highest observed success under Pass@$N$, demonstrating that incorporating strategy-specific structure can, in some cases, provide an additional benefit beyond generic compute awareness. Budget sweeps additionally show larger gains over CE at $256$ than at $16$ expansions for both tested strategies. Together, these findings show that considering the downstream search during training can improve performance.

\paragraph{Limitations and Future Work.}
Our experiments use one model and short reference proofs, with training restricted to offline SFT. UA accounts for compute without modeling the search rule; search-aware CAT models selected search dynamics on the demonstrated trace. Both remain trace-only. Additional training-time compute could sample one-step deviations, explore local trees around demonstrations, and increasingly approximate deployment search. One-step samples reveal alternative transitions; deeper searches provide evidence about alternative proofs and recovery costs. Such observations could refine search-aware objectives beyond fixed allocation. Testing whether these refinements improve performance, and how to allocate training compute between supervision and search, remain future work.
\clearpage

\clearpage

\appendix

\section{Strategy Images}
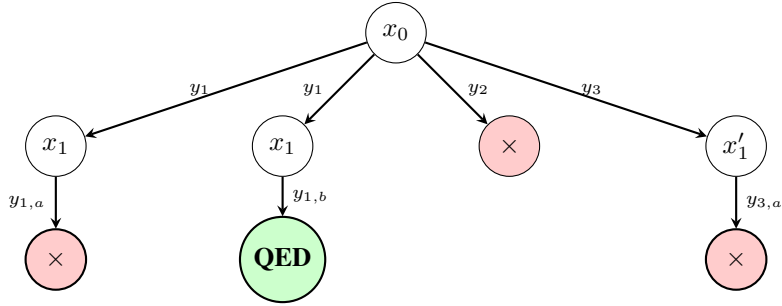
\begin{figure}[ht]
\centering
\begin{tikzpicture}[level 1/.style={sibling distance=30mm},level 2/.style={sibling distance=15mm},level distance=15mm,state/.style={circle, draw, minimum size=8mm, inner sep=1pt},dead/.style={circle, draw, fill=red!20, minimum size=8mm},qed/.style={circle, draw, fill=green!20, minimum size=8mm, font=\bfseries},edge/.style={->, >=stealth, thick}]
\node[font=\bfseries] at (0, 1.2) {Budget: $N=4$ parallel threads};

\node[state] (root) at (0,0) {$x_0$}
    child { node[state] (s1a) {$x_1$} 
        child { node[dead] (d1) {$\times$} edge from parent[edge] node[left, font=\scriptsize] {$y_{1,a}$} }
        edge from parent[edge] node[left=2pt, font=\scriptsize] {$y_1$}
    }
    child { node[state] (s1b) {$x_1$} 
        child { node[qed] (q1) {QED} edge from parent[edge] node[right, font=\scriptsize] {$y_{1,b}$} }
        edge from parent[edge] node[left=2pt, font=\scriptsize] {$y_1$}
    }
    child { node[dead] (d2) {$\times$} edge from parent[edge] node[right=2pt, font=\scriptsize] {$y_2$} }
    child { node[state] (s2) {$x'_1$}
        child { node[dead] (d3) {$\times$} edge from parent[edge] node[right, font=\scriptsize] {$y_{3,a}$} }
        edge from parent[edge] node[right=2pt, font=\scriptsize] {$y_3$}
    };
    
\end{tikzpicture}

\caption{Pass@N Parallel Search. $N=4$ independent paths explore the tree simultaneously. A single valid path reaching \texttt{QED} constitutes a global success.}
\label{fig:pass_at_n}
\end{figure}
\begin{figure}[ht]
    \centering
    \begin{tikzpicture}[
        level 1/.style={sibling distance=45mm},
        level 2/.style={sibling distance=25mm},
        level distance=18mm,
        state/.style={circle, draw, minimum size=8mm, inner sep=1pt},
        dead/.style={circle, draw, fill=red!20, minimum size=8mm},
        qed/.style={circle, draw, fill=green!20, minimum size=8mm, font=\bfseries},
        edge/.style={->, >=stealth, thick},
        jump/.style={->, >=stealth, thick, dashed, blue}
    ]
    
    \node[font=\bfseries] at (0, 1.2) {Best-First Node Expansions};

    \node[state, label=above:{\scriptsize $S=1.0$}] (root) at (0,0) {$x_0$}
        child { node[state, label=left:{\scriptsize $S=0.8$}] (s1) {$x_1$} 
            child { node[state, label=left:{\scriptsize $S=0.3$}] (s2) {$x_2$} edge from parent[edge] node[left=2pt, font=\scriptsize] {\textcircled{2} $y_{1,a}$} }
            child { node[dead] (d1) {$\times$} edge from parent[edge] node[right=2pt, font=\scriptsize] {\textcircled{3} $y_{1,b}$} }
            edge from parent[edge] node[left=5pt, font=\scriptsize] {\textcircled{1} $y_1$}
        }
        child { node[state, label=right:{\scriptsize $S=0.6$}] (s3) {$x'_1$}
            child { node[qed] (q1) {QED} edge from parent[edge] node[right=2pt, font=\scriptsize] {\textcircled{5} $y_3$} }
            edge from parent[edge] node[right=5pt, font=\scriptsize] {\textcircled{4} $y_2$}
        };

    \draw[jump] (s2.south) to[out=270, in=225, looseness=1.2] 
        node[pos=0.5, below=12pt, font=\scriptsize, align=center] {Global frontier re-evaluation:\\$S(x'_1) > S(x_2)$} 
        (s3.south west);

    \end{tikzpicture}
    \caption{Policy-Guided Best-First Search. The numbered circles indicate chronological node expansions. Expanding $x_0$ yields two valid states on the frontier, $x_1$ ($S=0.8$) and $x'_1$ ($S=0.6$). The algorithm greedily expands $x_1$ (step 2). However, generating the tactic to reach $x_2$ drops the path's geometric mean to $S=0.3$. Because the algorithm evaluates the global frontier, the decision rule abandons the active path and dynamically jumps across the tree to expand $x'_1$ (step 4), ultimately finding the proof.}
    \label{fig:levin_search}
\end{figure}
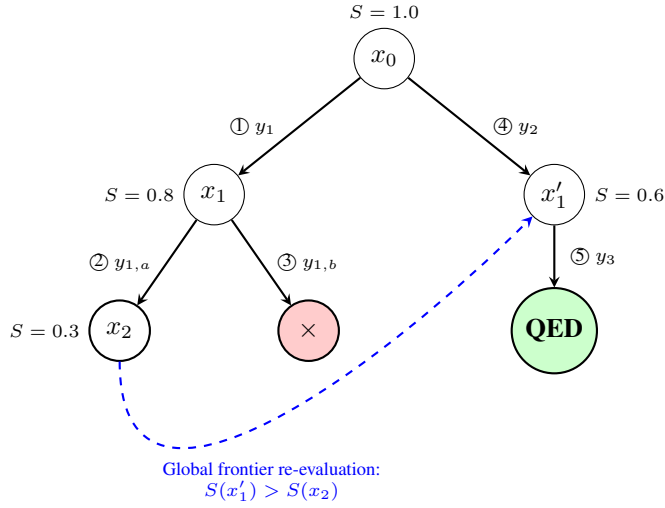

\begin{figure}[ht]
    \centering
    \begin{tikzpicture}[
        level 1/.style={sibling distance=40mm},
        level 2/.style={sibling distance=30mm},
        level 3/.style={sibling distance=25mm},
        level distance=15mm,
        state/.style={circle, draw, minimum size=8mm, inner sep=1pt},
        dead/.style={circle, draw, fill=red!20, minimum size=8mm},
        qed/.style={circle, draw, fill=green!20, minimum size=8mm, font=\bfseries},
        edge/.style={->, >=stealth, thick}
    ]
    
    \node[font=\bfseries] at (0, 1.2) {Global Budget: $N=5$ sequential expansions};

    \node[state] (root) at (0,0) {$x_0$}
        child { node[state] (s1) {$x_1$} 
            child { node[dead] (d1) {$\times$} edge from parent[edge] node[left=2pt, font=\scriptsize] {\textcircled{2} $y_{1,a}$} }
            child { node[state] (s2) {$x_2$}
                child { node[dead] (d2) {$\times$} edge from parent[edge] node[left=2pt, font=\scriptsize] {\textcircled{4} $y_{2,a}$} }
                child { node[qed] (q1) {QED} edge from parent[edge] node[right=2pt, font=\scriptsize] {\textcircled{5} $y_{2,b}$} }
                edge from parent[edge] node[right=2pt, font=\scriptsize] {\textcircled{3} $y_{1,b}$}
            }
            edge from parent[edge] node[left=5pt, font=\scriptsize] {\textcircled{1} $y_1$}
        };

    \end{tikzpicture}
    \caption{Depth-First Search. The numbered circles indicate the chronological order of node expansions. The algorithm pushes forward, backtracks locally upon failure (e.g., from step 2 back to $x_1$), and terminates the search immediately upon reaching the \texttt{QED} state at step 5.}
    \label{fig:dfs_search}
\end{figure}
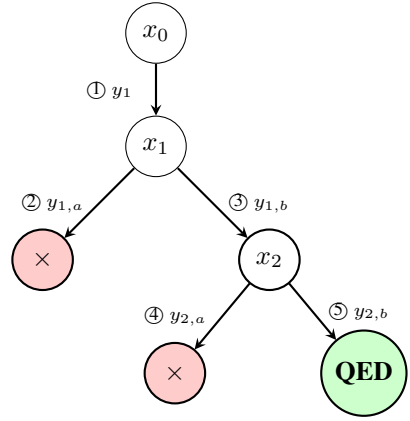

\begin{figure}[ht]
\centering
\begin{tikzpicture}[
    level 1/.style={sibling distance=40mm},
    level 2/.style={sibling distance=30mm},
    level 3/.style={sibling distance=20mm},
    level distance=15mm,
    state/.style={circle, draw, minimum size=8mm, inner sep=1pt},
    dead/.style={circle, draw, fill=red!20, minimum size=8mm},
    qed/.style={circle, draw, fill=green!20, minimum size=8mm, font=\bfseries},
    edge/.style={->, >=stealth, thick},
    jump/.style={->, >=stealth, thick, dashed, blue}
]

\node[font=\bfseries] at (0, 1.2) {Randomized Backtracking};

\node[state] (x0) at (0,0) {$x_0$}
    child { node[state] (x1) {$x_1$} 
        child { node[state] (x2) {$x_2$}
            child { node[dead] (d1) {$\times$} edge from parent[edge] node[left=2pt, font=\scriptsize] {\textcircled{3} $y_3$} }
            edge from parent[edge] node[left=2pt, font=\scriptsize] {\textcircled{2} $y_2$}
        }
        child { node[qed] (q1) {QED} edge from parent[edge] node[right=2pt, font=\scriptsize] {\textcircled{5} $y'_{2}$} }
        edge from parent[edge] node[left=2pt, font=\scriptsize] {\textcircled{1} $y_1$}
    };
    
\draw[jump, overlay] (d1.west) to[out=170, in=190, looseness=2.2] 
    node[left=4pt, font=\scriptsize, align=right] {\textcircled{4} Uniform random\\backtrack to $x_1$} 
    (x1.west);

\end{tikzpicture}
\caption{Randomized Depth-First Search. After expanding a deep, dead-end path (steps 1--3), the algorithm selects a node uniformly at random from the active path $\{x_0, x_1, x_2\}$ to backtrack to. In this case, it randomly selects $x_1$ (step 4), escaping the trap at $x_2$ and successfully finding the proof on an alternative branch (step 5).}
\label{fig:dfs_rand}
\end{figure}
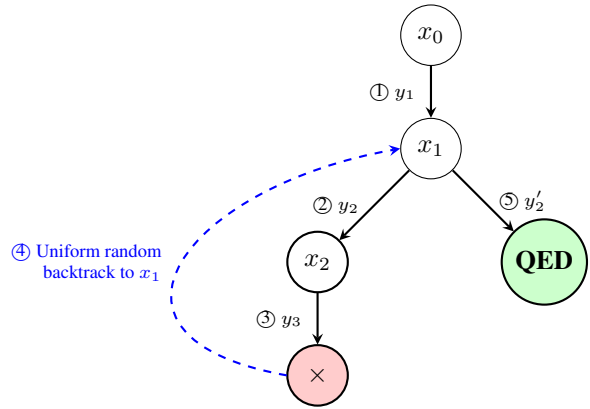
\begin{figure}[ht]
    \centering
    \begin{tikzpicture}[
        level 1/.style={sibling distance=40mm},
        level 2/.style={sibling distance=30mm},
        level 3/.style={sibling distance=20mm},
        level distance=18mm,
        state/.style={circle, draw, minimum size=8mm, inner sep=1pt},
        dead/.style={circle, draw, fill=red!20, minimum size=8mm},
        qed/.style={circle, draw, fill=green!20, minimum size=8mm, font=\bfseries},
        edge/.style={->, >=stealth, thick},
        jump/.style={->, >=stealth, thick, dashed, blue}
    ]

    \node[font=\bfseries] at (0, 1.2) {Value-Guided Backtracking};

    \node[state, label=above:{\scriptsize $S=1.0$}] (x0) at (0,0) {$x_0$}
        child { node[state, label=left:{\scriptsize $S=0.8$}] (x1) {$x_1$} 
            child { node[state, label=left:{\scriptsize $S=0.3$}] (x2) {$x_2$}
                child { node[dead] (d1) {$\times$} edge from parent[edge] node[left=2pt, font=\scriptsize] {\textcircled{3} $y_3$} }
                edge from parent[edge] node[left=2pt, font=\scriptsize] {\textcircled{2} $y_2$}
            }
            child { node[qed] (q1) {QED} edge from parent[edge] node[right=2pt, font=\scriptsize] {\textcircled{5} $y'_{2}$} }
            edge from parent[edge] node[left=2pt, font=\scriptsize] {\textcircled{1} $y_1$}
        };
        
    \draw[jump, overlay] (d1.west) to[out=170, in=190, looseness=2.2] 
        node[left=4pt, font=\scriptsize, align=right] {\textcircled{4} Score-weighted random\\backtrack to $x_1$\\($P \propto 0.8$)} 
        (x1.west);

    \end{tikzpicture}
    \caption{Value Guided Depth-First Search (VDFS). After expanding into a dead end (step 3), the algorithm evaluates the active path prefix $\{x_0, x_1, x_2\}$. Instead of backtracking uniformly, it samples an ancestor proportional to its Levin score $S$. In this case, it randomly jumps to $x_1$ (step 4) because its higher score ($S=0.8$) makes it a heavily favored target over the deep trap at $x_2$ ($S=0.3$). This allows it to successfully escape and find the proof on an alternative branch (step 5).}
    \label{fig:VDFS_search}
\end{figure}
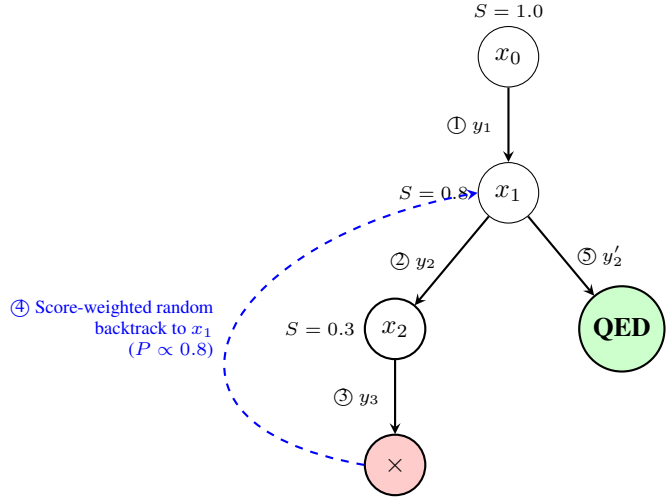

\begin{figure}[ht]
    \centering
    \begin{tikzpicture}[
        level 1/.style={sibling distance=48mm},
        level 2/.style={sibling distance=26mm},
        level distance=17mm,
        state/.style={circle, draw, minimum size=8mm, inner sep=1pt},
        dead/.style={circle, draw, fill=red!20, minimum size=8mm},
        qed/.style={circle, draw, fill=green!20, minimum size=8mm, font=\bfseries},
        edge/.style={->, >=stealth, thick},
        jump/.style={->, >=stealth, thick, dashed, blue},
        backup/.style={->, >=stealth, thick, dashed, gray}
    ]

    \node[font=\bfseries] at (0, 1.3) {MCTS};

    \node[state] (x0) at (0,0) {$x_0$}
        child { node[state, label=above:{\scriptsize $V{=}0.80$}] (x1) {$x_1$}
            child { node[dead] (x2) {$\times$}
                edge from parent[edge]
                node[left=2pt, font=\scriptsize] {\textcircled{3} $y_{1,a}$} }
            edge from parent[edge]
            node[left=5pt, font=\scriptsize] {\textcircled{1} $y_1$} }
        child { node[state, label=above:{\scriptsize $V{=}0.60$}] (x1p) {$x_1'$}
            child { node[qed] (qed) {QED}
                edge from parent[edge]
                node[right=2pt, font=\scriptsize] {\textcircled{4} $y_{2,b}$} }
            edge from parent[edge]
            node[right=5pt, font=\scriptsize] {\textcircled{2} $y_2$} };

    \draw[backup, overlay] (x2.south) to[out=250, in=280, looseness=1.6]
        node[below=2pt, font=\scriptsize, align=center]
        {backup: $Q(x_1)\!\leftarrow\!0.40,\ n(x_1)\!=\!2$}
        (x1.south);

    \draw[jump, overlay] (x1.east) to[out=-20, in=200, looseness=0.9]
        node[pos=0.5, below=4pt, font=\scriptsize, align=center]
        {UCT re-evaluation:\\$\mathrm{UCT}(x_1')=1.12 > \mathrm{UCT}(x_1)=0.77$}
        (x1p.west);

    \end{tikzpicture}
    \caption{Deployed Monte Carlo Tree Search. Circled numbers indicate the
    chronological order of node expansions; a budget of $N=4$ expansions is
    exhausted in this run. \emph{Selection/Expansion:} the root is expanded
    twice, sampling $y_1$ into $x_1$ (step~1) and $y_2$ into $x_1'$ (step~2);
    each is evaluated once via the policy's own Levin score, giving
    $V(x_1)=0.80$ and $V(x_1')=0.60$. With both children visited once, UCT
    (exploration constant $c=0.5$) scores $x_1$ at $1.22$ against $x_1'$ at
    $1.02$, so the search commits to the higher-value child and expands it
    (step~3), reaching a dead end. \emph{Backup:} the failure backs up a value
    of $0$ through $x_1$, dropping its running average to $Q(x_1)=0.40$.
    \emph{Selection reverses:} recomputing UCT with the updated statistics
    gives $x_1'$ a score of $1.12$ against $x_1$'s now-diminished $0.77$ --
    unlike Policy-Guided Best-First Search (\Cref{fig:levin_search}), where the
    frontier is re-ranked by a static score, here it is the shrinking
    exploration bonus \emph{and} the value drop from a real backup that
    overturn the earlier choice. The search abandons $x_1$'s subtree, expands
    $x_1'$ (step~4), and reaches \texttt{QED}.}
    \label{fig:mcts_search}
\end{figure}
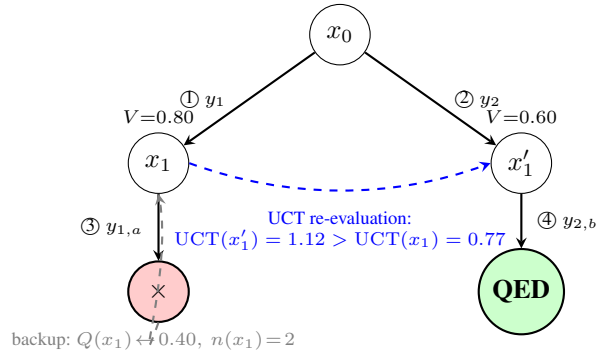

\clearpage
\section{Formal MDP Formulation of Automated Theorem Proving}
\label{app:mdp_formulation}

We formalize the structure of LLM-guided theorem proving in interactive environments like Lean, where text generation is strictly mediated by a formal verifier (the Lean kernel), as a Markov Decision Process (MDP) \citep{polu2020generativelanguagemodelingautomated, lample2022hypertree, yang2023leancopilot}. We define this sequential decision-making process as a tuple $\mathcal{M} = \langle \mathcal{X}, \mathcal{Y}, \mathcal{T}, \mathcal{R}, x_0 \rangle$:
\begin{itemize}
    \item \textbf{State Space ($\mathcal{X}$):} The discrete set of all valid Lean proof states. This space includes two terminal absorbing states: a universal success state $x_{QED}$ and a failure state $x_{err}$.
    \item \textbf{Action Space ($\mathcal{Y}$):} The space of possible tactics the LLM can generate.
    \item \textbf{Transition Function ($\mathcal{T}$):} The deterministic update rules governed by the Lean kernel, defined as $\mathcal{T}: \mathcal{X} \times \mathcal{Y} \to \mathcal{X}$. Given a state $x_{t-1}$ and tactic $y_t$, $\mathcal{T}(x_{t-1}, y_t) = x_t$ if the tactic is logically and syntactically valid. If $y_t$ is invalid, the kernel immediately rejects it, and $\mathcal{T}(x_{t-1}, y_t) = x_{err}$. If $y_t$ completes the proof, $\mathcal{T}(x_{t-1}, y_t) = x_{QED}$.
    \item \textbf{Reward Function ($\mathcal{R}$):} A sparse terminal reward function mapping to $\{0, 1\}$. $\mathcal{R}(x_{t-1}, y_t) = 1$ if $\mathcal{T}(x_{t-1}, y_t) = x_{QED}$, and $0$ otherwise.
    \item \textbf{Initial State ($x_0$):} The underlying theorem declaration to be proven.
\end{itemize}

The probability of executing any valid trajectory of length $K$, denoted $\tau = (x_0, y_1, x_1, \dots, y_K, x_K)$, is the product of the model's autoregressive action probabilities:
\begin{equation}
    P_{\pi_\theta}(\tau \mid x_0) = \prod_{t=1}^{K} \pi_\theta(y_t \mid x_{t-1})
\end{equation}

Let $\Gamma_{succ}(x_0)$ denote the intractable set of all valid, successful trajectories originating from $x_0$ and terminating in $x_{QED}$. The true probability of proof discovery without search (Pass@1) is the marginal likelihood over this entire set:
\begin{equation}
    P_{\mathrm{any}} := \sum_{\tau \in \Gamma_{succ}(x_0)} P_{\pi_\theta}(\tau \mid x_0)
\end{equation}
Because marginalizing over $\Gamma_{succ}(x_0)$ is computationally impossible, standard SFT sidesteps this global objective, reducing the objective to maximizing the likelihood of a single optimal trace $\tau^*$ via behavioral cloning, which yields the standard cross-entropy loss discussed in Section \ref{sec:formal_cat}.
\clearpage

\section{Approximation Error and Scaling}
\label{sec:scaling}

This appendix develops the approximation analysis for the \emph{search-aware} objectives introduced in \Cref{sec:strategies}. These objectives replace the deployed success probability $J^{\mathcal A}$ with a trace-supported proxy $\widehat J_{\mathcal A}$. We study the resulting approximation through the per-tactic gradient weights, asking two questions: (i) at a fixed budget, when is the search-aware weight closer to the deployed weight than CE's constant unit weight? (ii) how does this approximation error change as the test-time budget grows? The search-agnostic UA objective is a separate budget-based construction, discussed in \Cref{app:meanfield}, and need not approximate the dynamics of a particular deployed search.

We study Pass@$N$ (\Cref{subsec:passn_analysis}) and the shared-budget retry model used for best-first search (\Cref{subsec:bfs_analysis}). For Pass@$N$, the budget counts independent rollouts, allowing us to isolate the effect of alternative successful trajectories omitted by the demonstrated trace. For the retry model, we examine how unsuccessful off-trace excursions consume a shared expansion budget. The latter results rely on explicit assumptions about excursion costs and recovery, rather than providing a complete characterization of deployed best-first search.

The analysis separates three effects:
\begin{itemize}
    \item \textbf{Bypass channel.} For Pass@$N$, under the proportional-reallocation convention and \Cref{assum:single_visit}, the deployed weight is no larger than the trace-supported weight, which is no larger than CE's unit weight (\Cref{thm:passn_squeeze}).

    \item \textbf{Trap channel.} In a trap-only retry model with a common deterministic miss cost, excursions that return to the same trace state reduce the effective budget. The resulting weight lies between the retry-surrogate weight and CE's unit weight (\Cref{prop:shared_budget_bracketing_revised}).

    \item \textbf{Budget scaling.} In the models studied here, large-budget weight error vanishes for Pass@$N$ and for recoverable shared-budget excursions. Permanent absorption can instead leave a nonzero error floor (\Cref{subsec:hump_floor}).
\end{itemize}

These results identify off-trace quantities that demonstrations alone do not determine. Incorporating them requires additional modeling assumptions or observations of search behavior. \Cref{app:offtrace} explores a simple use of such assumptions through priors on off-trace cost and recovery.

\subsection{Well-posed deployed weights}
\label{subsec:scaling_elasticity}

The true objective depends on the entire policy, not only on the demonstrated probabilities $p=(p_1,\ldots,p_L)$, so "$\partial J/\partial p_t$" is meaningless until we fix how the remaining probability mass moves when $p_t$ does. We adopt the assumption of \emph{proportional reallocation}: at each demonstrated state, write
\begin{equation}
    \pi_\theta(y\mid x_{t-1}) = p_t\,\mathbf{1}\{y=y_t^*\} + (1-p_t)\,\rho_t(y),
    \qquad \rho_t \in \Delta\big(\mathcal Y\setminus\{y_t^*\}\big),
    \label{eq:reparam}
\end{equation}
and define derivatives in $p_t$ holding $\rho_t$, all other $p_s$, and all conditionals at non-demonstrated states fixed. This is not explicitly true, tactics similar to the demonstrated tactic will gain probability mass while distant ones will lose mass; however, this convention makes analysis feasible. This is strictly a convention and not a statement about how changing the gradient redistributes mass to unobserved tactics. All statements about the true weights below are relative to this convention. For any success function $J>0$ define the \emph{elasticity}
\begin{equation}
    w_t(J) := \frac{\partial \log J}{\partial \log p_t}
    = \frac{p_t}{J}\frac{\partial J}{\partial p_t},
    \label{eq:search_elasticity}
\end{equation}
so that $w_t^\star := w_t(J^{\mathcal A})$ is the (unknown) deployed weight, $\widehat w_t := w_t(\widehat J_{\mathcal A})$ is the CAT weight of \Cref{prop:cat_template}, and CE corresponds to $w_t^{\mathrm{CE}}\equiv 1$. Comparing training rules then reduces to a scalar question at each step.

\begin{proposition}[When is CAT closer than CE?]
\label{prop:cat_vs_sft_scalar}
Suppose $0\le \widehat w_t<1$. Then $\big|w_t^\star-\widehat w_t\big|\le \big|w_t^\star-1\big|$ if and only if $w_t^\star \le \tfrac{1}{2}(1+\widehat w_t)$. If $\widehat w_t=1$, CAT and CE coincide.
\end{proposition}

\begin{proof}
Square both sides, cancel $(w_t^\star)^2$, and divide by $1-\widehat w_t>0$.
\end{proof}

Thus, CAT wins exactly when the deployed search makes the demonstrated tactic less critical than single-shot generation assumes; CE can win only when the tactic remains a near single-shot bottleneck under deployment. The next two subsections determine which side of the crossover each strategy occupies.

\subsection{Pass@\texorpdfstring{$N$}{N}: unit elasticity and a one-sided squeeze}
\label{subsec:passn_analysis}

Let $a(p)=\prod_{i=1}^{L}p_i$ be the one-rollout probability of the demonstrated trace, $b(p)$ the one-rollout probability of proving the theorem by any \emph{other} successful trajectory, and $P_{\mathrm{any}}=a+b$.\footnote{Note that $\pi_{\mathrm{proof}}$ in \Cref{subsec:passn} equals $a$.} With $F(x)=1-(1-x)^N$, the deployed objective is $J^{\mathrm{true}}=F(P_{\mathrm{any}})$ and the surrogate is $\widehat J = F(a)$. Define the elasticity of the outer function,
\begin{equation}
    H_F(x)=\frac{xF'(x)}{F(x)}.
\end{equation}

\begin{lemma}[Properties of $H_F$]
\label{lem:HF}
For $F(x)=1-(1-x)^N$ and $x\in(0,1)$,
\begin{equation}
    H_F(x)=\frac{N}{\sum_{j=0}^{N-1}(1-x)^{-j}},
\end{equation}
so $H_F$ is nonincreasing, $H_F(x)\le 1$ with $H_F(0^+)=1$, and for fixed $x$, $H_F(x)\to 0$ as $N\to\infty$.
\end{lemma}

\begin{proof}
$F(x)=x\sum_{k=0}^{N-1}(1-x)^{k}$ and $F'(x)=N(1-x)^{N-1}$; divide and shift the index by $N-1$. Each summand $(1-x)^{-j}\ge 1$ and is nondecreasing in $x$, giving the bound and monotonicity; the sum diverges geometrically in $N$ for fixed $x$, giving the limit.
\end{proof}

\begin{assumption}[At most one visit]
\label{assum:single_visit}
Almost surely, every successful one-rollout trajectory visits each demonstrated proof state $x_{t-1}$ at most once. This holds automatically when states are history-augmented (the search tree), which is the case in our setting.
\end{assumption}

\begin{lemma}[Unit elasticity]
\label{lem:unit_elasticity}
Under \Cref{eq:reparam} and \Cref{assum:single_visit}, partition the non-demonstrated success mass at step $t$ as $b = b_t^{\mathrm{sh}}+b_t^{\mathrm{sib}}+b_t^{\varnothing}$: trajectories that traverse the demonstrated transition $(x_{t-1},y_t^*)$ but are not $\tau^*$; trajectories that visit $x_{t-1}$ and deviate there; and trajectories that never visit $x_{t-1}$. Then
\begin{equation}
    p_t\frac{\partial P_{\mathrm{any}}}{\partial p_t}
    = \big(a+b_t^{\mathrm{sh}}\big) - \frac{p_t}{1-p_t}\,b_t^{\mathrm{sib}}
    \;\le\; P_{\mathrm{any}},
    \qquad\text{i.e.}\qquad
    \eta_t := \frac{p_t\,\partial_{p_t}P_{\mathrm{any}}}{P_{\mathrm{any}}} \le 1.
    \label{eq:unit_elasticity}
\end{equation}
No demonstrated tactic can move the total success probability excessively. ($\eta_t$ may be negative when sibling mass dominates.)
\end{lemma}

\begin{proof}
By \Cref{assum:single_visit} each successful trajectory contains the conditional at $x_{t-1}$ at most once. Under \Cref{eq:reparam} its probability is therefore of the form $p_t C$, $(1-p_t)\rho_t(y)C$, or $C$ with $C$ free of $p_t$, according to which subset of $b$ it belongs. Differentiate termwise and multiply by $p_t$. The inequality follows since the first term is at most $P_{\mathrm{any}}$ and the second is nonpositive.
\end{proof}

\begin{theorem}[Weight squeeze for Pass@$N$]
\label{thm:passn_squeeze}
Under \Cref{assum:single_visit}, for every step $t$ and budget $N$,
\begin{equation}
    w_t^\star \;=\; \eta_t\,H_F(P_{\mathrm{any}})
    \;\le\; H_F(P_{\mathrm{any}})
    \;\le\; H_F(a) \;=\; \widehat w_t \;\le\; 1 \;=\; w_t^{\mathrm{CE}}.
    \label{eq:passn_squeeze}
\end{equation}
Consequently $|\widehat w_t - w_t^\star| \le |1-w_t^\star|$: CAT weakly dominates CE at every step and every budget, with strict domination whenever any inequality is strict.
\end{theorem}

\begin{proof}
$w_t^\star = H_F(P_{\mathrm{any}})\,\eta_t$ and $\widehat w_t = H_F(a)\cdot 1$ by the chain rule and $p_t\partial_{p_t}a = a$. Apply \Cref{lem:unit_elasticity}, then $a\le P_{\mathrm{any}}$ with $H_F$ nonincreasing, then $H_F\le 1$ (\Cref{lem:HF}). Domination: $w_t^\star\le\widehat w_t\le 1$ gives $\widehat w_t-w_t^\star\le 1-w_t^\star$.
\end{proof}

\paragraph{Error decomposition.} The gap localizes into the two ways the surrogate is blind:
\begin{equation}
    \widehat w_t - w_t^\star
    = \underbrace{\big[H_F(a)-H_F(P_{\mathrm{any}})\big]}_{\text{ignored bypass mass}}
    \;+\; \underbrace{(1-\eta_t)\,H_F(P_{\mathrm{any}})}_{\text{sharing / reallocation structure}}.
    \label{eq:passn_error_decomp}
\end{equation}
The first term is the objective-level blindness to $b$; the second is present even at $N=1$ and reflects that raising $p_t$ proportionally starves sibling alternatives at $x_{t-1}$ while riding along with shared-prefix alternatives.

\paragraph{Remark.} A useful strengthening is assuming $b\le\varepsilon a$, the mean value theorem bounds the first term of \Cref{eq:passn_error_decomp} by $\varepsilon a \sup_{x\in[a,a+b]}|H_F'(x)|$, giving a quantitative guarantee. This assumption is in a sense that the expert trace takes up a nontrivial chunk of the probability mass of correct proofs.

\subsection{Best-first search: censoring and excursions}
\label{subsec:bfs_analysis}

For BFS the surrogate is the retry objective of \Cref{eq:bfs_trace_objective},
\begin{equation}
    \widehat J(N)=\mathbb P(S_L\le N),
    \qquad S_L=\sum_{t=1}^{L}T_t,
    \qquad T_t\sim\mathrm{Geom}(p_t)\ \text{independent}.
    \label{eq:retry_surrogate_again}
\end{equation}
Unlike Pass@$N$, a deployed miss does not cost exactly one budget unit. The search may wander off-trace before the trace node is re-selected, or it may find a proof out there without returning. We separate these two events, then approximate each.

\subsubsection{The CAT weight is a criticality}

\begin{lemma}[CAT weight as criticality]
\label{lem:criticality_revised}
Let $T_t'$ be an independent copy of $T_t$ and $p_i\in(0,1)$. For \Cref{eq:retry_surrogate_again},
\begin{equation}
    \widehat w_t(N)
    = 1-\frac{\mathbb P(S_L+T_t'\le N)}{\mathbb P(S_L\le N)}
    = \mathbb P\!\left(S_L\le N < S_L+T_t' \,\middle|\, S_L\le N\right).
    \label{eq:criticality_revised}
\end{equation}
Hence $0\le\widehat w_t(N)\le 1$, $\widehat w_t(L)=1$, $\widehat w_t(N)\to 0$ as $N\to\infty$, and $\widehat w_t(N)$ is non-increasing in $N$.
\end{lemma}

\begin{proof}
For $T\sim\mathrm{Geom}(p)$ and any bounded $\psi$, $p\,\partial_p\mathbb E[\psi(T)]=\mathbb E[\psi(T)]-\mathbb E[\psi(T+T')]$. Apply with $\psi(k)=\mathbb P(R_t+k\le N)$, $R_t=\sum_{i\ne t}T_i$, and divide by $\widehat J(N)$. For monotonicity: the pmf of $S_L$ is a convolution of log-concave pmfs, hence log-concave, and so is its CDF $F$; thus $F(n-k)/F(n)$ is nondecreasing in $n$ for fixed $k\ge 1$, and $1-\widehat w_t(n)=\mathbb E_{T_t'}[F(n-T_t')/F(n)]$ is nondecreasing.
\end{proof}

\subsubsection{Separating alternative proofs: the bypass censored coupling}
To analyze the impact of off-trace proofs on BFS, consider an artificial "bypass censored process". This process behaves exactly like standard, deployed BFS. Except, if it ever samples a tactic which would yield a proof other than the expert proof, we replace the observation of $x_{QED}$ by $x_{err}$. So it only succeeds if it returns the demonstrated trace but it is allowed to be off path during the search.

\begin{lemma}[Bypass Censoring]
\label{lem:censoring}
 Let $N$ be the initial budget, $J^{\mathrm{trap}}(N)$ be the censored process' success probability and let $\beta(N)$ be the probability that the deployed run samples a transition which would have been censored if we had instead run the censored process. Then
\begin{equation}
    J^{\mathrm{trap}}(N)\;\le\;J^{\mathrm{true}}(N)\;\le\;J^{\mathrm{trap}}(N)+\beta(N).
    \label{eq:censoring_bracket}
\end{equation}
\end{lemma}

\begin{proof}
Couple the two runs on the same tactic draws and selector randomness. They are identical until the deployed run first samples a transition which would have been censored: at that moment, the deployed run has succeeded. If no censored transition is ever sampled, outcomes agree. So censored success implies deployed success, and deployed success implies censored success or divergence.
\end{proof}

This is just to say that alternate proof paths can only \emph{add} success. To use this to analyze the weights, split the true success probability by which event occurs first, $J^{\mathrm{true}} = J^{\dagger} + B$ ($J^\dagger$: completion of the demonstrated trace; $B$: completion of an alternate trace). From looking at the sum of elasticities,
$$\frac{\partial J^{\mathrm{true}}}{\partial p_t} = \frac{\partial J^{\dagger}}{\partial p_t} + \frac{\partial B}{\partial p_t}$$
$$w_t^\star = \frac{p_t}{J^{\mathrm{true}}} \frac{\partial J^{\dagger}}{\partial p_t} + \frac{p_t}{J^{\mathrm{true}}} \frac{\partial B}{\partial p_t}$$
$$w_t^\star = \left( \frac{J^{\dagger}}{J^{\mathrm{true}}} \right) \left( \frac{p_t}{J^{\dagger}} \frac{\partial J^{\dagger}}{\partial p_t} \right) + \left( \frac{B}{J^{\mathrm{true}}} \right) \left( \frac{p_t}{B} \frac{\partial B}{\partial p_t} \right)$$
\begin{equation}
    w_t^\star
    = \lambda\, w_t\!\big(J^{\dagger}\big) + (1-\lambda)\, w_t(B),
    \qquad \lambda = \frac{J^{\dagger}}{J^{\mathrm{true}}},
    \label{eq:mixture_identity}
\end{equation}
holds \emph{exactly}.

This mixture identity reveals that the true weight is subject to a tug-of-war between two competing "channels" of success.

First, the \textbf{trap channel} ($J^\dagger$) captures trajectories that wander into dead ends but eventually return to complete the demonstrated proof. Because these off-trace excursions consume a shared compute budget, they increase the bottleneck on the remaining path. This increased criticality pushes the weight $w_t(J^\dagger)$ \emph{up} relative to the trace-supported surrogate $\widehat w_t$ (formalized in \Cref{prop:shared_budget_bracketing_revised}).

Conversely, the \textbf{bypass channel} ($B$) captures trajectories that successfully discover alternative, unobserved proofs. The existence of these alternative solutions makes the exact demonstrated tactic less critical for global success, which pushes the weight \emph{down}, mirroring the elasticity squeeze observed in the Pass@$N$ analysis. Thus, we will focus the remaining analysis on the effect of the trap channel.

One structural caveat remains: the true BFS selector's geometric-mean scores dynamically depend on $p$, a dependence the trace-supported surrogate erases to maintain independence between steps. Because Levin scores are cumulative, increasing an early on-trace probability $p_s$ creates a cascading advantage for all subsequent trace nodes. This allows the trace to win frontier comparisons sooner, actively suppressing the length of future off-trace excursions. This positive feedback loop means the true search penalizes on-trace mistakes even more heavily than our surrogate models. Consequently, the trap channel's upward pressure on the deployed weight $w_t^\star$ is strictly stronger in reality. We will, however, treat the mixture identity (\Cref{eq:mixture_identity}) as exact.

\subsubsection{The Trap Channel: Exploration Waste and Shared-Budget Bracketing}
\label{subsubsec:trap_channel}

In the trap channel, when BFS samples an off-trace tactic at state $x_{t-1}$, the algorithm embarks on an off-trace excursion until it returns back to the trace. To quantify the computational toll of deviations, let $E_{t,j} \in \{0, 1, \ldots\} \cup \{\infty\}$ denote the number of wasted node expansions spent off-trace following the $j$-th miss at state $x_{t-1}$ before the search backtracks and re-selects the trace node ($E_{t,j} = \infty$ if the algorithm is permanently absorbed). The trap-channel success probability is then precisely represented by:
\begin{equation}
    J^{\mathrm{trap}}(N) = \mathbb{P}\!\left(\sum_{t=1}^{L}T_t \;+\; \sum_{t=1}^{L}\sum_{j=1}^{T_t-1}E_{t,j} \;\le\; N\right),
    \label{eq:excursion_representation}
\end{equation}
where $T_t \sim \mathrm{Geom}(p_t)$ tracks on-trace trials. Different choices of weight during training can be viewed as distinct assumptions on $E$:
\begin{itemize}
    \item \textbf{Minimal waste ($E \equiv 0$):} Every off-trace tactic is instantly rejected by the environment ($x_{err}$). A miss costs exactly $1$ expansion, recovering our CAT retry surrogate exactly ($J^{\mathrm{trap}}(N) = \widehat{J}(N)$).
    \item \textbf{Permanent absorption ($E \equiv \infty$):} Any off-trace mistake traps the search indefinitely. Here, $J^{\mathrm{trap}}(N) = \prod_{t=1}^L p_t$ for any budget $N \ge L$, yielding unit elasticity ($w_t^\star = 1$), recovering the exact regime where standard SFT weighting is optimal.
\end{itemize}

Between these two extremes lies the realistic setting where off-trace excursions consume finite, non-zero compute before backtracking. To isolate how this exploration waste alters the true weight, we first examine the case of a fixed return cost.

\begin{assumption}[Deterministic excursion cost]
\label{assum:effective_deviation_cost_revised}
In the trap-channel process, every miss incurs a total expenditure of $\kappa \ge 1$ expansions (the initial trial plus an excursion of length $E \equiv \kappa - 1$) before returning the search to $x_{t-1}$; successful trials cost $1$ expansion; $\kappa$ is independent of the policy probabilities $p$.
\end{assumption}

\begin{proposition}[Shared-budget bracketing]
\label{prop:shared_budget_bracketing_revised}
Under \Cref{assum:effective_deviation_cost_revised}, let $N\ge L$ be an integer budget and $p_s\in(0,1)$ for all $s$. Completing the demonstrated path in the bypass-censored process within budget $N$ is equivalent to evaluating the retry surrogate at a reduced effective budget:
\begin{equation}
    J^{\mathrm{trap}}(N) = \widehat{J}(N_{\mathrm{eff}}),
    \qquad
    N_{\mathrm{eff}} = L + \frac{N-L}{\kappa} \le N.
    \label{eq:effective_budget_revised}
\end{equation}
For noninteger budget arguments $B$, interpret $\widehat{J}(B)=\widehat{J}(\lfloor B\rfloor)$ and $\widehat{w}_t(B)=\widehat{w}_t(\lfloor B\rfloor)$.

Define the trap-channel weight by $w_t^{\mathrm{trap}}(N):=w_t(J^{\mathrm{trap}}(N))$. Because $\kappa$ is independent of $p$, the effective-budget identity and \Cref{lem:criticality_revised} give
\begin{equation}
    \widehat{w}_t(N)
    \le w_t^{\mathrm{trap}}(N)
    = \widehat{w}_t(N_{\mathrm{eff}})
    \le 1 = w_t^{\mathrm{CE}}.
    \label{eq:shared_budget_bracket_revised}
\end{equation}
\end{proposition}

\begin{proof}
Completing step $t$ requires $G_t \sim \mathrm{Geom}(p_t)$ trials costing $\kappa(G_t-1)+1$ total expansions. Summing over the trace yields a total cost $\kappa S_L - (\kappa-1)L \le N \iff S_L \le N_{\mathrm{eff}}$. The bracket directly follows since $N_{\mathrm{eff}} \le N$ for $\kappa \ge 1$ and $\widehat{w}_t(\cdot)$ is non-increasing.
\end{proof}

In real policy-guided Best-First Search, excursions do not incur deterministic costs; rather, they exhibit stochastic, bounded trajectories ($E \le \bar{\kappa} - 1$ a.s.). Crucially, we can actively engineer this bound during deployment to prevent permanent absorption.

\paragraph{Enforcing Bounded Excursions via Global Search Restarts.}
In practical ATP deployments, we actively prevent policies from being trapped in cyclical or deep, unguided proof attempts ($P(E = \infty) = 0$) by deploying a restart-based Best-First Search heuristic. Specifically, if the global BFS does not discover a complete proof within a step threshold of $M < N$ total node expansions, the active search tree is aborted and the algorithm resets entirely to the initial theorem state $x_0$ to begin a fresh run with the remaining compute budget. However, the cost to setting such an $M$ too low is aborting a possibly prominent path so in practice this $M$ is fairly high. This restart mechanism guarantees that no single exploratory trajectory can waste more than $M$ expansions off-trace. Consequently, the exploration waste incurred before re-attempting the demonstrated trajectory is strictly bounded ($E \le M$ a.s., enforcing a maximum excursion severity $\bar{\kappa} \le M + 1$). While enforcing this test-time constraint eliminates the SFT ceiling ($w_t^\star < 1$), combining it with \Cref{prop:shared_budget_bracketing_revised} reveals an irreducible tension between trace-only supervision and search optimization.

\begin{proposition}[Irreducible Ambiguity under Bounded Excursions]
\label{prop:irreducible_ambiguity}
Call a per-step weighting rule \emph{trace-supported} if it depends solely on training-time observables along $\tau^*$ (specifically, the demonstrated probabilities $\{p_s\}_{s=1}^L$, the search strategy $\mathcal{A}$, and the initial budget $N$). Suppose test-time search engineering enforces an upper bound on excursion waste such that $\kappa \in [1, \bar{\kappa}]$ for a finite limit $\bar{\kappa} > 1$. Then:
\begin{enumerate}
    \item[(i)] The true deployed weight $w_t^\star$ is strictly confined to the compressed bracket:
    \begin{equation}
        w_t^\star \in \left[\,\widehat{w}_t(N),\; \widehat{w}_t\!\left(L + \frac{N-L}{\bar{\kappa}}\right)\right].
    \end{equation}
    \item[(ii)] Because the realized excursion waste $\kappa$ depends on unobserved off-trace dynamics, no trace-supported objective can uniquely identify $w_t^\star$. Any trace-supported rule $w_t$ incurs a worst-case approximation error bounded below by half the bracket width:
    \begin{equation}
        \sup_{\kappa \in [1, \bar{\kappa}]} \big|w_t - w_t^\star(\kappa; N)\big| \;\ge\; \frac{1}{2}\left( \widehat{w}_t\!\left(L + \frac{N-L}{\bar{\kappa}}\right) - \widehat{w}_t(N) \right) \;>\; 0.
    \end{equation}
\end{enumerate}
\end{proposition}

\begin{proof}
By \Cref{prop:shared_budget_bracketing_revised}, an environment with effective excursion cost $\kappa$ realizes the deployed weight $w_t^\star = \widehat{w}_t(L + \frac{N-L}{\kappa})$. Because environments with $\kappa=1$ (instant syntax error) and $\kappa=\bar{\kappa}$ (maximum restart cap $M$) yield identical training-time observables along $\tau^*$, any trace-supported rule must output the same scalar $w_t$ across both environments. The minimax lower bound follows directly from applying the triangle inequality to the interval endpoints.
\end{proof}

\subsection{Budget scaling of the weight error: hump or floor}
\label{subsec:hump_floor}

Using the understanding we've built, we can understand how the error in our approximation scales. A key intuition we will need is that the surrogate's weight error is generally not monotone in $N$. When $N$ is minimal ($N=1$ in our setting), we are just performing a single rollout. We aren't really performing search, and both the surrogate and real weights understand this, so the error of the surrogate is 0. Conversely, as $N\to\infty$, our probability of success (with some caveats) becomes 1. Thus, both the surrogate and deployed weights tend to 0. The error only really occurs in the middle between these two extremes.

\paragraph{The minimal-budget endpoint is universal.} At the smallest budget that can possibly succeed, any strategy must execute $\tau^*$ with no misses: both $J^{\mathcal A}$ (trap channel) and $\widehat J_{\mathcal A}$ equal $\prod_t p_t$, whose elasticity is $1$ at every step. So $\widehat w_t = w_t^\star = 1 = w_t^{\mathrm{CE}}$ and all three rules coincide (alternate proofs shorter than $L$ are the only exception, entering through $B$ in \Cref{eq:mixture_identity}).

\begin{proposition}[Pass@$N$: vanishing error]
\label{prop:passn_asymptotics}
Fix the policy with $a\in(0,1)$. Then $\widehat w_t(1)=1$ and $w_t^\star(1)=\eta_t$; both $\widehat w_t(N)$ and $w_t^\star(N)$ tend to $0$ as $N\to\infty$; the bypass term of \Cref{eq:passn_error_decomp} is zero at $N=1$ and vanishes as $N\to\infty$ (a hump), while the reallocation term decays monotonically from $1-\eta_t$. In particular, the CAT error tends to $0$ while the CE error $|1-w_t^\star|$ tends to $1$.
\end{proposition}

\begin{proof}
At $N=1$, $H_F\equiv 1$. As $N\to\infty$, $H_F(x)\to 0$ for fixed $x>0$ (\Cref{lem:HF}); apply to $x=a$ and $x=P_{\mathrm{any}}$ in \Cref{thm:passn_squeeze} and \Cref{eq:passn_error_decomp}.
\end{proof}

For BFS, the outcome depends on the excursions. We isolate it to a specific case.

\begin{assumption}[Absorbing-excursion model]
\label{assum:absorbing_excursions}
In the trap-only censored process, each miss at step $t$, independently of everything else, launches a non-returning excursion ($E=\infty$) with probability $q_t\in[0,1]$, and otherwise an excursion of bounded length $E\le\bar\kappa-1$.
\end{assumption}

\begin{proposition}[Hump or floor]
\label{prop:hump_floor}
Under \Cref{assum:absorbing_excursions}:
\begin{enumerate}
    \item[(a)] \textbf{(Recurrence $\Rightarrow$ hump.)} If $q_t=0$ for all $t$ and excursions are deterministic ($E\equiv\kappa-1$), the trap-channel error is $w_t^\star(N)-\widehat w_t(N) = \widehat w_t(N_{\mathrm{eff}})-\widehat w_t(N)\ \ge 0$, which is $0$ at $N=L$, tends to $0$ as $N\to\infty$ (since $N_{\mathrm{eff}}\to\infty$), and is positive in between whenever $\kappa>1$: the error is humped. For bounded random excursions with no absorption, the error is likewise zero at $N=L$ and vanishes as $N\to\infty$; its finite-budget behavior is governed by \Cref{eq:excursion_representation}.
    \item[(b)] \textbf{(Absorption $\Rightarrow$ floor.)} If $q_t>0$, then as $N\to\infty$,
    \begin{equation}
        J^{\mathrm{trap}}(N)\;\longrightarrow\;\prod_{t=1}^{L}\frac{p_t}{p_t+(1-p_t)q_t},
        \qquad
        w_t^\star(N)\;\longrightarrow\;\frac{q_t}{p_t+(1-p_t)q_t}\;>\;0,
        \label{eq:bias_floor}
    \end{equation}
    while $\widehat w_t(N)\to 0$: the CAT error converges to the strictly positive floor \eqref{eq:bias_floor}. As $q_t\to1$, the floor tends to $1$ and CE becomes asymptotically exact at step $t$, recovering the permanent-absorption limit; as $q_t\to0$, the limiting error at that step vanishes.
\end{enumerate}
\end{proposition}

\begin{proof}
(a) is \Cref{prop:shared_budget_bracketing_revised} plus \Cref{lem:criticality_revised}. (b) With unlimited budget, finite excursions are free, so step $t$ succeeds iff $y_t^*$ is sampled before an absorbing excursion: per trial, advance w.p.\ $p_t$, absorb w.p.\ $(1-p_t)q_t$, retry otherwise, giving per-step success $f(p_t)=p_t/(p_t+(1-p_t)q_t)$, independent across steps. Direct computation gives $\partial\log f/\partial\log p = q/(p+(1-p)q)$. The interchange of limit and derivative is justified within the model by monotone convergence of $J^{\mathrm{trap}}(N)\uparrow\prod_t f(p_t)$ together with monotonicity of $w_t^\star(N)$ inherited from \Cref{lem:criticality_revised} applied at the extended-cost representation.
\end{proof}

\paragraph{Approximation is Incorrect Asymptotically, but Weights Converge.} The pieces we don't account for in our surrogate (the alternate proof mass $\beta(N)$ and total excursion expansions) do grow with $N$. But their effect on the weights is self-limiting. Both objectives saturate near $1$, and every elasticity is near $0$, so no weighting rule can be far from the truth. This is a key advantage of test-time aligned training.

\clearpage

\section{Depth First Search Variants}
\label{app:DFS}
Depth First Search (DFS) is not a standard algorithm used in policy guided search. We introduce it in this context because analyzing it and its variants serve as excellent examples about the mechanics of constructing search aligned loss functions.

\subsection{Standard DFS} 
Standard DFS utilizes its computational budget to sequentially construct a single path. When the algorithm expands a node $x_{t-1}$ and the tactic is rejected ($\mathcal{T}(x_{t-1}, y_t) = x_{err}$), the search does not abandon the entire trajectory. Instead, the decision rule dictates a retreat of exactly one step to the immediate parent state $x_{t-1}$, expending another unit of compute to sample an alternative tactic. This localized trial-and-error continues until the algorithm either successfully reaches $x_{QED}$ or exhausts the global node expansion budget $N$.

Finally, we note that because standard DFS commits rigidly to a valid path, its chronologically local backtracking is a structural weakness. If the search ventures into a deep, valid, but useless branch, the localized backtracking cannot escape the broader strategic error. The LLM will simply exhaust its entire computational budget trapped in the depths of a spurious path.

\paragraph{Instantiation.}
Define $\widehat{c} = \langle x, b \rangle$, where $x$ is the active trace state and $b$ is the remaining node expansion budget. The trace-supported frontier is always a singleton, $\widehat{\mathcal{F}} = \{x\}$. 

The positive update advances the path and consumes one expansion. For the negative update, an off-trace tactic immediately collapses the search back to the parent $x_{t-1}$ at the cost of one expansion. Thus, the updates are:
\[
    \widehat{U}_{\mathrm{DFS}}^{+}(x_{t-1}, b) = (x_t, b-1), \qquad \widehat{U}_{\mathrm{DFS}}^{-}(x_{t-1}, b) = (x_{t-1}, b-1)
\]
Because the frontier is always a single node, the selector is trivial: $S_{\mathrm{DFS}}(x \mid \widehat{c}) = 1$. Consequently, the trace-supported recurrence for DFS perfectly mirrors the recurrence for BFS (\Cref{eq:bfs_trace_recurrence}):
\[
    \widehat{J}_{\mathrm{DFS}}(b,t) = p_t \widehat{J}_{\mathrm{DFS}}(b-1,t+1) + (1-p_t) \widehat{J}_{\mathrm{DFS}}(b-1,t).
\]

Thus, the closed form loss will collapse to that which we derived for BFS.

\subsection{Randomized Depth-First Search (RDFS)} 
To mitigate the vulnerability of DFS to spurious paths, we introduce a randomized variant. RDFS disrupts this failure mode by altering the backtracking target. When the policy $\pi_\theta$ reaches a dead end, the algorithm does not simply return to $x_{t-1}$. Instead, it selects a state \textit{uniformly at random} from the entire active path prefix $(x_0, x_1, \dots, x_{t-1})$ and resumes expansion from there. This allows the search to stochastically jump out of deep traps, efficiently reallocating compute to higher-level branching points while preserving the depth-seeking advantages of the standard algorithm. A visualization is given in \Cref{fig:dfs_rand}.

\paragraph{Instantiation.}
Let $\widehat{c} = \langle t, b \rangle$, with the frontier tracking the full active path $\widehat{\mathcal{F}} = \{1, 2, \dots, t\}$. 
To incorporate randomized backtracking, the randomness is injected directly into the selector. When a dead end is reached, rather than the update rule deterministically collapsing the path, the negative update triggers the selector to sample a state uniformly at random from the entire available frontier $\widehat{\mathcal{F}}$. 

\paragraph{Trace-supported recurrence.}
Because of this randomized selector, the compute consumed at depth $t$ cannot be modeled as a simple sequence of independent geometric random variables. Instead, the trace-supported objective forms a dynamic program. 

Let $\widehat{J}_{\mathrm{RDFS}}(b, t)$ denote the exact probability of discovering the complete proof given that the search is currently evaluating the trace state at index $t$ with a remaining budget $b$. The boundary conditions are:
\begin{align}
    \widehat{J}_{\mathrm{RDFS}}(b, L+1) &= 1 \quad \text{for all } b \ge 0 \quad \text{(Proof successfully completed)} \\
    \widehat{J}_{\mathrm{RDFS}}(0, t) &= 0 \quad \text{for all } t \le L \quad \text{(Compute budget exhausted)}
\end{align}

For $b > 0$ and $t \le L$, the recurrence relation resolving \Cref{eq:trace_supported_recurrence} branches into two outcomes. With probability $p_t$, the model generates the correct tactic and advances. With probability $(1 - p_t)$, the model fails, and the selector backtracks uniformly to any valid ancestor index $k \in \{1, \dots, t\}$. Thus, the recurrence becomes:
\begin{equation}
    \widehat{J}_{\mathrm{RDFS}}(b, t) = p_t \widehat{J}_{\mathrm{RDFS}}(b-1, t+1) + (1 - p_t) \frac{1}{t} \sum_{k=1}^{t} \widehat{J}_{\mathrm{RDFS}}(b-1, k)
\end{equation}

Unlike Pass@$N$ or BFS, this objective does not reduce to a standard closed-form probability distribution.

\subsection{Value Guided Depth-First Search (VDFS)}
\label{sec:Vdfs}
To mitigate the vulnerability of DFS while leveraging the policy's own internal confidence, we introduce Value Guided DFS (VDFS). Unlike RDFS, which backtracks uniformly, VDFS makes an informed retreat. When the search reaches a dead end, it backtracks to an ancestor state in the active path prefix with probability proportional to that state's Levin score (from \Cref{eq:bfs_score}). By using the geometric mean of the policy's probabilities as a confidence metric, the search preferentially returns to the most promising points on the current path. This concentrates the remaining computational budget on states the policy deems most reliable, while maintaining enough stochasticity to escape deep, spurious branches. A visualizaiton can be seen in \Cref{fig:VDFS_search}.

\paragraph{Instantiation.}
Let $\widehat{c} = \langle t, b \rangle$, where $t$ is the active trace index (corresponding to the state $x_{t-1}$), and $b$ is the remaining node expansion budget. The trace-supported frontier tracks the full active path prefix of states: $\widehat{\mathcal{F}} = \{x_0, x_1, \dots, x_{t-1}\}$.

To compute the backtrack distribution, we evaluate the Levin score for each state on the frontier. Let $s_k$ denote the Levin score of state $x_k$. For the root, $s_0 = 1$. For $k \ge 1$, the score is the geometric mean of the on-trace probabilities leading to that state:
\begin{equation}
    s_k = \left( \prod_{j=1}^{k} p_j \right)^{\frac{1}{k}}
\end{equation}

The positive update advances the proof by one step: $\widehat{U}_{\mathrm{VDFS}}^{+}(t, b) = (t+1, b-1)$.
For the negative update, an off-trace tactic consumes a unit of budget and triggers a backtrack. The selector samples a backtrack target state $x_k \in \widehat{\mathcal{F}}$ with probability proportional to its score $s_k$. Transitioning to $x_k$ sets the new active trace index to $k+1$. Thus, the negative update maps stochastically:
\begin{equation}
    \widehat{U}_{\mathrm{VDFS}}^{-}(t, b) = (k+1, b-1), \quad \text{where } k \sim P(k) = \frac{s_k}{\sum_{i=0}^{t-1} s_i}
\end{equation}

\paragraph{Trace-supported recurrence.}
Let $\widehat{J}_{\mathrm{VDFS}}(b, t)$ denote the exact probability of discovering the complete proof from the trace state at index $t$ with remaining budget $b$. The boundary conditions are identical to RDFS:
\begin{align}
    \widehat{J}_{\mathrm{VDFS}}(b, L+1) &= 1 \quad \text{for all } b \ge 0 \quad \text{(Proof successfully completed)} \\
    \widehat{J}_{\mathrm{VDFS}}(0, t) &= 0 \quad \text{for all } t \le L \quad \text{(Compute budget exhausted)}
\end{align}

For $b > 0$ and $t \le L$, the recurrence resolving \Cref{eq:trace_supported_recurrence} weights the backtrack sum by the Levin scores. With probability $p_t$, the model generates the correct tactic and advances. With probability $(1 - p_t)$, the model fails, and the selector backtracks to state $x_k$ ($k \in \{0, \dots, t-1\}$) according to the score distribution. The recurrence becomes:
\begin{equation}
    \widehat{J}_{\mathrm{VDFS}}(b, t) = p_t \widehat{J}_{\mathrm{VDFS}}(b-1, t+1) + (1 - p_t) \sum_{k=0}^{t-1} \frac{s_k}{\sum_{i=0}^{t-1} s_i} \widehat{J}_{\mathrm{VDFS}}(b-1, k+1)
\end{equation}

Like RDFS, this objective doesn't have a closed form. Notably, the gradients flow not only through the direct local transitions but also through the Levin scores $s_k$, which are themselves functions of the policy probabilities. This couples the training objective, forcing the model to calibrate its on-path probabilities to optimize the resulting backtracking distribution.
\clearpage

\section{Monte Carlo Tree Search (MCTS)}
\label{sec:mcts}

The strategies considered so far commit their budget by a rule fixed in advance: DFS always retries the deepest node, RDFS backtracks uniformly, and VDFS backtracks in proportion to a static per-node score. Monte Carlo Tree Search \citep{kocsis2006bandit, coulom2006efficient} instead allocates compute according to statistics it accumulates as the search runs. We include it because it is the only strategy in our taxonomy whose selector depends on the \emph{history} of the search rather than on the current configuration alone, and it therefore tests our construction at its least convenient.

MCTS is a 4-phase process at deployment. At each node, \emph{Selection} either commits to its most promising child or stops to gain more information about another tactic at the current node. The choice to widen is governed by a progressive-widening schedule: a node that has been visited $n$ times is permitted $\lceil\sqrt{n+1}\rceil$ distinct tactic samples \citep{coulom2007computing}. \emph{Expansion} samples one tactic from $\pi_\theta$ at the selected node and submits it to the kernel, consuming one node expansion. \emph{Evaluation} assigns the resulting state a value; in the variant, we choose to use the policy's own Levin score \Cref{eq:bfs_score} rather than a learned critic, so that MCTS requires no component the other strategies don't utilize. \emph{Backup} propagates that value to every ancestor on the selected path, updating the statistics that drive future selections. A visualization is given in \Cref{fig:mcts_search}.

\paragraph{Instantiation.}
With abuse of notation, let a configuration be
\[
    \widehat{c} \;=\; \big\langle\, (n_k, m_k, W_k)_{k=0}^{L-1},\; b \,\big\rangle,
\]
where $n_k$ is the visit count of the state $x_k$, $m_k$ the number of tactics already sampled there, $W_k$ the accumulated backed-up value, and $b$ the remaining expansion budget. The trace-supported frontier is the prefix $\{x_0,\ldots,x_{L-1}\}$. The selector is not stochastic: the descent stops at the shallowest node that is either allowed a new tactic by the widening schedule or has no child yet,
\begin{equation}
\begin{aligned}
    d(\widehat{c}) &\;=\;
    \min\Big\{\,k \;:\; m_k < \big\lceil\sqrt{n_k+1}\big\rceil
    \;\;\text{or}\;\; x_{k+1}\ \text{not yet reached}\,\Big\}, \\
    &\qquad
    \widehat{S}_{\mathrm{MCTS}}(x_k \mid \widehat{c}) = \mathbf{1}\{k = d(\widehat{c})\}.
\end{aligned}
\label{eq:mcts_selector}
\end{equation}
The updates apply the expansion and the backup together. Writing $d = d(\widehat{c})$, a demonstrated tactic advances the trace and backs up the value $V(x_{d+1})$ of the state it reaches, while a miss backs up zero:
\begin{align}
    \widehat{U}_{\mathrm{MCTS}}^{+}(\widehat{c})
    &:\quad m_d \mathrel{+}= 1,\quad
    n_k \mathrel{+}= 1,\; W_k \mathrel{+}= V(x_{d+1})\ \ \forall k \le d,\quad
    b \mathrel{-}= 1,
    \label{eq:mcts_uplus}\\
    \widehat{U}_{\mathrm{MCTS}}^{-}(\widehat{c})
    &:\quad m_d \mathrel{+}= 1,\quad
    n_k \mathrel{+}= 1,\; W_k \mathrel{+}= 0\ \ \forall k \le d,\quad
    b \mathrel{-}= 1,
    \label{eq:mcts_uminus}
\end{align}
with solved configurations (for the trace supported version) those in which depth $L$ has been reached and dead configurations those with $b = 0$. Both updates decrement the budget, so the recurrence terminates.

For BFS, the naive selector is degenerate on
a demonstrated trace because a chain never offers two frontier candidates at once. The same erasure happens here, in a slightly different place. The UCT rule ranks the children of a node by an exploration-adjusted value estimate, but under the trace transition \Cref{eq:trace_transition} every node has at most one child, so there is never a comparison for it to make. What survives is only the widening schedule in \Cref{eq:mcts_selector}: the rule that decides how many tactic samples a node receives before the search commits to descending. Thus in practice, this strategy trains like a local Pass@$N$ at each node.

\clearpage

\section{Strategy Matching Within the Search-Aware Family}

The main comparison evaluates search-aware CAT and search-agnostic UA across deployed strategies. Here, we isolate a different question: within the search-aware family, does matching the training objective to the deployed strategy improve proof success? Using the fixed-budget setup of \Cref{subsec:exp_matrix}, we evaluate the Pass@$N$ and BFS search-aware adapters under both deployment strategies, with CE as a reference.

\begin{table}[ht!]
\centering
\caption{\textbf{Strategy matching within the search-aware family.} Proof-success rates (\%) at $N=256$ tactic expansions. Columns are deployed strategies; rows are training objectives. Bold marks the higher point estimate among the two search-aware objectives in each column.}
\label{tab:matrix_abl_2}
\small
\begin{tabular}{lcc}
\toprule
Training objective & Pass@$N$ & BFS \\
\midrule
CE                         & 21.2 & 23.8 \\
Search-aware CAT (Pass@$N$) & \textbf{26.6} & 24.5 \\
Search-aware CAT (BFS)      & 23.7 & \textbf{25.1} \\
\bottomrule
\end{tabular}
\end{table}

The matched objective has the higher observed success rate under each deployment strategy. This supports strategy matching within the tested search-aware pair, complementing the comparison with shared UA in \Cref{tab:matrix}.
\clearpage
\section{Experimental Details}
\label{app:experimental_details}

\subsection{Common setup}
\label{app:details_common}

\paragraph{Environment.}
All experiments run against Lean~4 through LeanDojo~$4.20.0$ \citep{yang2023leandojo}. The policy is prompted with the pretty-printed proof state and emits a single tactic, which is submitted to the Lean kernel; the kernel verdict determines the transition. One \emph{node expansion} denotes one policy sample together with its kernel verification, and is the unit in which all budgets are expressed. Kernel interaction is serial within a theorem, so wall-clock cost per unsolved theorem grows linearly in the budget.

\paragraph{Data.}
Training and evaluation both draw on \texttt{cat-\allowbreak searcher/\allowbreak leandojo-\allowbreak benchmark-\allowbreak 4-\allowbreak random}, a random split of mathlib4 \citep{mathlib2020}. The alignment stage consumes $40{,}660$ traced proofs comprising $76{,}003$ (state, tactic) training examples. The evaluation pool is the $458$ held-out theorems whose reference proofs contain between $2$ and $5$ tactic steps; there is no overlap between training and evaluation theorems. Reference proofs are used only to compute training weights and to define the pool. At evaluation, a theorem counts as solved if the search returns any kernel-verified proof, whether or not it matches the reference.

\paragraph{Model and adapters.}
The policy is Qwen2.5-Math-7B-Instruct \citep{qwen2024qwen25} in bfloat16. All fine-tuning uses LoRA \citep{hu2022lora} with rank $r=16$, $\alpha=32$, dropout $0.05$, applied to the query, key, value, output, gate, up, and down projections. This trains $40{,}370{,}176$ of $7{,}655{,}986{,}688$ parameters ($0.53\%$). Gradient checkpointing is enabled throughout. All runs use a single NVIDIA A100-40GB. Training runs used FlashAttention-2 \citep{dao2023flashattention2}.

\subsection{Training procedure and hyperparameters}
\label{app:details_training}

\paragraph{Stage 1: Warm-up.}
A single supervised checkpoint is the common ancestor of every model we report. The first set of experiments branch off of two warm up epochs of SFT; while the second set branch off of only one. This discrepancy is a result of human error; however, given that runs are not compared cross-experiment, we find it reasonable to keep.

\paragraph{Stage 2: Alignment.}
From the warm-up checkpoint the runs branch. The control receives one further epoch of standard cross-entropy SFT. Each aligned model receives one further epoch over the same examples in the same order under the same optimizer settings, with the only difference being the scalar weight $\widehat{w}_t^{\mathcal{A}}$ applied to each tactic's cross-entropy gradient. An alignment epoch is $19{,}001$ optimizer steps and takes approximately $2.4$\,h.

\begin{table}[ht]
\centering
\caption{Optimization hyperparameters, identical for the control and every aligned run.}
\label{tab:hparams}
\vspace{3pt}
\small
\begin{tabular}{ll}
\toprule
Optimizer & AdamW \\
Learning rate & $1\times10^{-4}$ \\
Schedule & linear decay, warm-up over $5\%$ of steps \\
Micro-batch size & $4$ \\
Gradient accumulation & $16$ \\
Effective batch size & $64$ \\
Gradient clipping & $1.0$ (global $\ell_2$ norm) \\
Maximum sequence length & $512$ tokens \\
Precision & bfloat16 with autocast \\
Alignment epochs & $1$ (on top of the shared warm-up) \\
Random seed & $42$ \\
\bottomrule
\end{tabular}
\end{table}

Non-finite losses are skipped without an optimizer step; this occurs rarely and is logged. Examples exceeding the maximum sequence length are dropped at dataset construction rather than truncated, so no partially-supervised tactic enters training. The closed-form weights are evaluated directly in log space for numerical stability. The weights defined by a dynamic program are obtained by autodiff.

\subsection{Inference and evaluation}
\label{app:details_inference}

\paragraph{Serving.}
Generation uses vLLM~$0.6.3$ with the LoRA adapter supplied per request, GPU memory fraction $0.85$, and maximum model length $4096$ tokens. Sampling temperature is $0.8$. The evaluation seed ($42$) fixes both the shuffle of the theorem pool and the sampling stream, so every configuration sees the identical theorems in the identical order; this is what licenses the paired analysis below and what makes a truncated cell a well-defined prefix of the full pool.

\begin{table}[ht]
\centering
\caption{Search parameters. Budgets are node expansions.}
\label{tab:search_params}
\vspace{3pt}
\small
\begin{tabular}{lll}
\toprule
Parameter & Experiment 1 & Experiment 2 \\
\midrule
Budget $N$ & $256$ & $16$, $64$, $256$, $1024$ \\
Best-first beam width & $2$ & $2$ \\
Expansions per pop & $4$ & $4$ \\
Best-first score exponent $\alpha$ & $0.5$ & $0.5$ \\
Pass@$N$ generation batch & $16$ & $64$ \\
Maximum proof depth & $5$ & $5$ \\
Per-strategy wall-clock cap & $90$\,s & none \\
Per-theorem wall-clock cap & $500$\,s & none \\
\bottomrule
\end{tabular}
\end{table}

\paragraph{Scoring and truncation.}
A theorem is scored as solved only if a kernel-verified proof is returned within the budget; unsolved-within-budget counts as a failure, with no exclusion. The two experiments differ in one respect that we record explicitly. Experiment 1 imposes the wall-clock caps in \Cref{tab:search_params}, so a theorem may in principle be cut off before its expansion budget is exhausted; Experiment 2 disables them, guaranteeing that every theorem spends its full budget. The caps are the reason the two experiments' absolute rates are not directly comparable, and they were removed for Experiment 2 because the effect grows with the budget: a wall-clock cap binds far more often at $N=1024$ than at $N=64$ and would otherwise have depressed exactly the high-budget cells whose behavior the experiment is about.

\paragraph{Statistics.}
Because all configurations are evaluated on the same theorems in the same order, every comparison is paired. For two configurations we report $b$, the number of theorems solved by the second and not the first, and $c$, the number solved by the first and not the second, on the subset $n$ that both evaluated. The reported difference is $(b-c)/n$; its $95\%$ interval uses the McNemar variance $\operatorname{Var}[(b-c)/n] = (b+c)/n^{2}$; and the reported $p$-value is the exact two-sided McNemar (sign) test on the $b+c$ discordant pairs. We use the paired difference rather than the difference of marginal rates throughout, because the two diverge whenever the cells being compared cover different numbers of theorems, and only the paired quantity is consistent with the interval and the $p$-value beside it. Where a figure shows unpaired intervals on individual rates (\Cref{fig:main_comparison}) this is stated in the caption; those intervals are wider than the paired ones and overlapping bars do not indicate a non-significant difference.

\subsection{Experiment 1: fixed-budget comparison}
\label{app:details_exp1}

\paragraph{Configurations.}
Six deployed strategies at $N=256$: Pass@$N$ and best-first search from \Cref{sec:strategies}, and DFS, RDFS, VDFS and MCTS from \Cref{app:DFS}. Each is paired with the CAT weight derived for it (same for BFS/DFS), and we additionally deploy the single shared UA adapter under every strategy.

\paragraph{Control.}
The control for this experiment is the epoch-matched model: warm-up plus one additional epoch of standard SFT. Control and aligned models therefore receive identical total optimization, and the comparison isolates the loss weighting.

\subsection{Experiment 2: budget sweep}
\label{app:details_exp2}

\paragraph{Budget-matched adapters.}
We train a separate adapter per evaluation budget with the weight's budget parameter set to the deployment value, and evaluate each at its own budget only.

\paragraph{Control.}
The control for this experiment is the shared warm-up checkpoint, with an additional SFT epoch.

\paragraph{Pools.}
The $N \in \{16, 64, 256, 1024\}$ are evaluated on the full $458$-theorem pool. However, we only allowed $80$ GPU hours so not all evaluations were fully completed.

\paragraph{Reproduction.}
Given the trained adapters, each cell is one invocation of the evaluation harness with the strategy, budget, and adapter path. Cells write one JSON log per strategy keyed by theorem name, and all tables and figures in this paper are generated directly from those logs.

\clearpage

\section{High-Budget Evaluation}
\label{app:sweep_1024}

To complement \Cref{subsec:exp_sweep}, we also evaluated budget-matched CAT adapters at $N=1024$ under Pass@$N$ and BFS, using the same training and evaluation protocol. The evaluation-time limit restricted the paired comparisons to $57$ theorems for Pass@$N$ and $56$ for BFS. \Cref{tab:sweep_1024} reports these results, and \Cref{fig:sweep_gain_full} shows the complete budget sweep.

\begin{table}[ht]
\centering
\caption{\textbf{Compute-limited evaluation at $N=1024$ node expansions.}
$n$ is the paired subset size, $b$ counts CAT-only successes, and $c$ counts control-only successes. Gains and $95\%$ intervals are in percentage points. Intervals and exact two-sided McNemar tests follow \Cref{tab:sweep}.}
\label{tab:sweep_1024}
\small
\begin{tabular}{llcccc}
\toprule
Strategy & Budget $N$ & $n$ & $b/c$ & Gain (pp), $95\%$ interval & $p$ \\
\midrule
Pass@$N$ & $1024$ & 57 & 6/1 & $+8.8\;[-0.3,17.9]$ & $0.125$ \\
BFS      & $1024$ & 56 & 5/2 & $+5.4\;[-3.9,14.6]$ & $0.453$ \\
\bottomrule
\end{tabular}
\end{table}

\begin{figure}[htbp]
\centering
\includegraphics[width=0.60\linewidth]{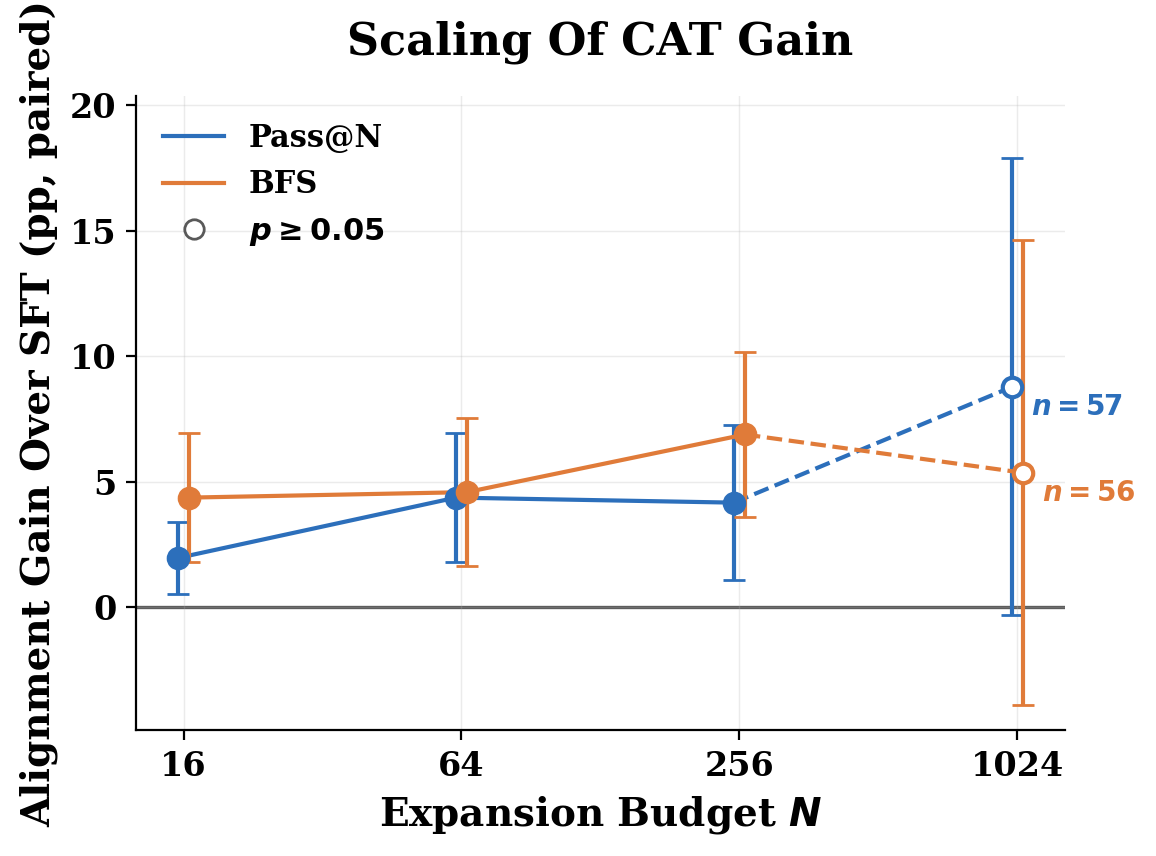}
\caption{Complete budget sweep, including the compute-limited
$N=1024$ evaluation. Error bars show $95\%$ intervals; filled markers
denote $p<0.05$ and hollow markers denote $p\geq0.05$. Dashed segments
connect to the $N=1024$ estimates, annotated with their paired subset
sizes. The completed subsets differ across budgets, so connecting
lines are descriptive rather than a common-cohort comparison.}
\label{fig:sweep_gain_full}
\end{figure}

\paragraph{Interpretation.}
Both high-budget point estimates favor CAT, but each comparison is
based on only seven discordant theorem pairs, and both reported
$95\%$ intervals include zero. These results are therefore inconclusive
about the magnitude of the alignment benefit at $N=1024$ and do not
establish a trend from $N=256$ to $N=1024$. We report them for
completeness rather than using them as evidence of increasing
alignment gains at large budgets.

\clearpage

\section{Priors on Off-Trace Behaviour}
\label{app:offtrace}

In \Cref{sec:scaling}, we characterized the gap between our trace-supported surrogate and the deployed search. However, we haven't deeply explored if we can use this to meaningfully improve training. We implied that placing priors on off-trace seems like a reasonable way to utilize this observation. We use this section to take a brief step in that direction and see if a simple prior on loss can be beneficial.

\subsection{A two-parameter prior}
\label{app:offtrace_family}

\Cref{prop:shared_budget_bracketing_revised} and \Cref{prop:hump_floor} show that off-trace behaviour enters the objective through two scalars:
\begin{itemize}
    \item $\kappa \ge 1$, the expected number of expansions a miss consumes, counting the off-trace excursion before the search returns to the trace;
    \item $q \in [0,1)$, the probability that a miss launches an excursion that never returns.
\end{itemize}
Both enter mechanically and in a strategy-independent way. \Cref{prop:shared_budget_bracketing_revised} states that a deployment in which a miss costs $\kappa$ expansions realises the surrogate at the deflated budget
\begin{equation}
    N_{\mathrm{eff}} \;=\; L + \frac{N - L}{\kappa} \;\le\; N,
    \label{eq:offtrace_neff}
\end{equation}
so $\kappa$ is applied by evaluating the strategy's own objective at, $N_{\mathrm{eff}}$ in place of $N$. Absorption enters the trace-supported recurrence \Cref{eq:trace_supported_recurrence} by splitting the miss branch: with probability $(1-p_t)q$ the search is absorbed and contributes nothing, and with probability $(1-p_t)(1-q)$ it continues as the uncorrected recurrence prescribes. Every occurrence of $(1-p_t)$ on a continuation branch is therefore replaced by $(1-p_t)(1-q)$, with the missing mass leaking to a terminal failure state.

Note that a better prior can make assumptions on $\kappa$ and $q$ as functions of p or prior probabilities, but we are merely going to test the simple case.

\subsection{Applying the prior to each strategy}
\label{app:offtrace_impl}

The prior is applied to the objective the surrogate implies for each strategy.

\paragraph{Best-first search.} The trace-supported objective for BFS is the retry recurrence \Cref{eq:bfs_trace_recurrence}. The prior is applied directly to that recurrence,
\begin{equation}
    \widehat{J}^{\kappa,q}(b,t)
    = p_t\,\widehat{J}^{\kappa,q}(b-1,\,t+1)
    + (1-p_t)(1-q)\,\widehat{J}^{\kappa,q}(b-1,\,t),
    \label{eq:offtrace_bfs_dp}
\end{equation}
with $\widehat{J}^{\kappa,q}(b,L)=1$, $\widehat{J}^{\kappa,q}(0,t)=0$, run to $b = N_{\mathrm{eff}} = L+(N-L)/\kappa$. 

\paragraph{RDFS and VDFS.} The same substitution applies to the dynamic programs of \Cref{app:DFS}: the budget axis runs to $N_{\mathrm{eff}}$ and the backtrack branch carries the factor $(1-q)$, leaving the uniform and value-weighted backtrack kernels otherwise unchanged.

\paragraph{MCTS.} Similarly, we just readjust the dynamic program in essentially the same fashion as the other search strategies.

\subsection{Setup}
\label{app:offtrace_setup}

The prior is set by hand. The primary setting is $\kappa = 5$, $q = 0.02$; under BFS we additionally train $\kappa = 3$, $q = 0$ to probe sensitivity. Everything else matches \Cref{subsec:exp_matrix}: the same warm-up checkpoint, one alignment epoch over the same examples in the same order under the same optimiser settings, a budget of $N = 256$ node expansions, and the full $458$-theorem pool. The only difference between a prior-corrected adapter and its uncorrected counterpart is the scalar weight applied to each tactic's cross-entropy gradient. Each adapter is deployed under the strategy whose objective it corrects.

\subsection{Results}
\label{app:offtrace_results}
\begin{table}[ht]
\centering
\caption{The off-trace prior applied to each strategy's own CAT objective
(proofs found, \%, $N{=}256$ node expansions, $458$ theorems). ``CAT'' is the
uncorrected weight of \Cref{sec:strategies,app:DFS}, i.e.\ the
$(\kappa,q)=(1,0)$ member of the family.}
\label{tab:offtrace_exact}
\vspace{3pt}
\small
\begin{tabular}{llccc}
\toprule
Strategy & Prior & SFT & CAT & $+$prior \\
\midrule
\multirow{2}{*}{BFS}
 & $\kappa{=}5,\ q{=}0.02$ & 23.8 & 25.1 & \textbf{25.8} \\
 & $\kappa{=}3,\ q{=}0$    & 23.8 & \textbf{25.1} & \textbf{25.1} \\
\midrule
RDFS & $\kappa{=}5,\ q{=}0.02$ & 20.7 & \textbf{23.6} & 22.7 \\
VDFS & $\kappa{=}5,\ q{=}0.02$ & 21.2 & 23.8 & \textbf{24.2} \\
MCTS & $\kappa{=}5,\ q{=}0.02$ & 20.3 & 22.5 & \textbf{22.9} \\
\bottomrule
\end{tabular}
\end{table}

\Cref{tab:offtrace_exact} reports the strategies for which the prior was applied. Generally, the priors edged out the basic CAT weights but not by a significant amount in any case. The conclusion we take from this is that perhaps with optimal tuning of the prior, we can get increased gains, but the value is marginal.
\clearpage

\section{Search-Agnostic Uniform-Allocation Objective: Additional Details}
\label{app:meanfield}

This appendix supplements the search-agnostic UA objective introduced in \Cref{subsec:mf_main}. UA accounts for the compute budget through a uniform local allocation without specifying the deployed search rule. We relate its weighting function to the shared-budget retry objective, providing an additional interpretation of how the two objectives account for compute.

\paragraph{From shared-budget to local criticality.}
Consider the retry objective in \Cref{eq:bfs_trace_objective}, with independent $T_t\sim\mathrm{Geom}(p_t)$, $p_t\in(0,1)$, and integer $N\ge L$. Write
\[
    S_L=\sum_{s=1}^{L}T_s,
    \qquad R_t=\sum_{s\ne t}T_s,
    \qquad B_t=N-R_t.
\]
Here, $B_t$ is the budget left for step $t$ after accounting for the waiting times of all other steps. For integer $k\ge1$, define the success probability and gradient weight of a single-tactic retry problem by
\begin{equation}
\begin{split}
    F(p,k)&=1-(1-p)^k,\\
    h(p,k)&=\frac{\partial\log F(p,k)}{\partial\log p}
    =\frac{k p(1-p)^{k-1}}{1-(1-p)^k}.
\end{split}
\label{eq:mf_local_criticality}
\end{equation}
The shared-budget retry weight then has the exact representation
\begin{equation}
    \widehat w_t^{\mathrm{retry}}(N)
    =\mathbb E\!\left[
        h(p_t,B_t)\,\middle|\,S_L\le N
      \right].
    \label{eq:mf_budget_average}
\end{equation}
To see this, extend $F(p,k)=0$ for $k\le0$ and write $\widehat{\mathcal J}^{\mathrm{retry}}(N) =\mathbb E[F(p_t,B_t)]$. The distribution of $B_t$ does not depend on $p_t$, so differentiating and dividing by the success probability weights each residual budget by $\mathbb P(B_t=k)F(p_t,k)/\widehat{\mathcal J}^{\mathrm{retry}}(N) =\mathbb P(B_t=k\mid S_L\le N)$, giving \Cref{eq:mf_budget_average}. Only $k\ge1$ occurs under this conditional distribution.

Thus, the shared-budget weight combines two ingredients: the local criticality $h(p_t,k)$ and a distribution over the budget available to that step. The other tactic probabilities enter through this budget distribution.

\paragraph{Uniform allocation and objective.}
The UA objective instead assigns each demonstrated step a uniform local allocation,
\begin{equation}
    \bar K=\frac{N}{L}.
    \label{eq:mf_kbar}
\end{equation}
This is a modeling choice, not an estimate of $\mathbb E[B_t\mid S_L\le N]$ or an average over deployed search strategies. It yields the factorized objective
\begin{equation}
    \widehat{\mathcal J}^{\,\mathrm{UA}}(\theta)
    =\prod_{t=1}^{L}\Big[1-(1-p_t)^{\bar K}\Big].
    \label{eq:mf_meanfield_J}
\end{equation}
For integer $\bar K$, this is the probability that every demonstrated step succeeds within its own quota of $\bar K$ independent attempts, without transferring unused quota between steps. Noninteger $N/L$ uses the continuous extension of the same formula. Applying the gradient identity in \Cref{eq:generic_cat_gradient} gives
\begin{equation}
    \widehat w_t^{\,\mathrm{UA}}
    =\frac{\bar K p_t(1-p_t)^{\bar K-1}}
           {1-(1-p_t)^{\bar K}}
    =h(p_t,\bar K).
    \label{eq:mf_meanfield_weight}
\end{equation}
Comparing \Cref{eq:mf_budget_average,eq:mf_meanfield_weight}, UA evaluates the local criticality at a fixed allocation, whereas the shared-budget retry objective averages it over the conditional distribution of the residual budget. The UA weight is the exact gradient weight of its own objective in \Cref{eq:mf_meanfield_J}. This connection does not require UA to approximate the dynamics of a particular deployed search algorithm.

\paragraph{Properties of the uniform-allocation weights.}
UA is search-agnostic, but not compute-agnostic. At $\bar K=1$, \Cref{eq:mf_meanfield_J} reduces to the demonstrated trace likelihood and all weights equal one, recovering CE. For fixed $\bar K\ge1$, the weight approaches one as $p_t\to0$; for fixed $p_t\in(0,1)$, it approaches zero as $\bar K\to\infty$. Thus, tactics that are unlikely to be found within the local budget retain high weight, whereas tactics that can already be recovered through repeated sampling receive less emphasis. Unlike the shared-budget retry weight, the UA weight does not depend on the other steps' consumption of a shared budget. At fixed $N$ and $L$, each weight depends only on $p_t$.

\paragraph{Interpretation.}
The shared-budget and UA objectives retain the same local retry criticality but differ in their treatment of compute allocation. As discussed in \Cref{sec:main_theory}, demonstrations alone do not generally determine deployed gradient weights because they omit alternative proofs and off-trace recovery behavior. The empirical comparison in \Cref{tab:matrix} evaluates the usefulness of search-aware CAT and search-agnostic UA for training; it does not establish which objective more accurately approximates the deployed gradient weights.

\clearpage
\section{LLM Disclosure}
LLMs were used in 3 primary ways. The first is essentially all experiments were set up and instantiated with LLM coding. We had a separate LLM check for correctness of the code. The second is when the paper was drafted, we repeatedly asked LLMs to critique and we incorporated its feedback if we deemed it to be appropriate. The third was finding related papers. We read all of these before ever citing them to confirm the relation to our work.
\end{document}